%% file: main.tex
\documentclass[10pt, conference]{IEEEtran}

\usepackage{lmodern}
\usepackage{amsmath,amssymb,amsfonts}
\usepackage{algorithmic}
\usepackage{graphicx}
\usepackage{textcomp}
\usepackage{multirow}
\usepackage{arydshln}
\usepackage{times}
\usepackage{soul}
\usepackage{caption}
\usepackage[colorlinks=true,citecolor=blue]{hyperref}
\usepackage{adjustbox}
\usepackage{hhline}
\usepackage{booktabs}
\usepackage{makecell}
\usepackage{microtype}
\usepackage{subfigure}
\usepackage{balance}

\usepackage[capitalise]{cleveref}

\usepackage[T1]{fontenc}
\usepackage{inconsolata}
 
\usepackage{amsthm}
\usepackage{bm}
\usepackage{float}
\usepackage{booktabs}
\usepackage[utf8]{inputenc}
\usepackage[ruled,linesnumbered]{algorithm2e}
\usepackage[dvipsnames]{xcolor}

\newtheorem{theorem}{Theorem}
\newtheorem{lemma}{Lemma}

\SetKwInput{KwInput}{Input}                
\SetKwInput{KwOutput}{Output}  

\definecolor{commentcolor}{RGB}{110,154,155}  
  
\newcommand{\PyCode}[1]{\ttfamily\textcolor{black}{#1}} 

\begin{document}

\title{\emph{Calibre}: Towards Fair and Accurate Personalized Federated Learning with Self-Supervised Learning}

\author{\IEEEauthorblockN{Sijia Chen, Ningxin Su, Baochun Li}
\IEEEauthorblockA{\textit{Department of Electrical and Computer Engineering} \\
\textit{University of Toronto}\\
Toronto, Ontario, Canada \\
\tt\small sjia.chen@mail.utoronto.ca, \tt\small ningxin.su@mail.utoronto.ca, \tt\small bli@ece.toronto.edu}
}

\maketitle

\begin{abstract}

In the context of personalized federated learning, existing approaches train a global model to extract transferable representations, based on which any client could train personalized models with a limited number of data samples. Self-supervised learning is considered a promising direction as the global model it produces is generic and facilitates personalization for all clients fairly. However, when data is heterogeneous across clients, the global model trained using SSL is unable to learn high-quality personalized models. In this paper, we show that when the global model is trained with SSL without modifications, its produced representations have fuzzy class boundaries. As a result, personalized learning within each client produces models with low accuracy. In order to improve SSL towards better accuracy without sacrificing its advantage in fairness, we propose \emph{Calibre}, a new personalized federated learning framework designed to calibrate SSL representations by maintaining a suitable balance between more generic and more client-specific representations. \emph{Calibre} is designed based on theoretically-sound properties, and introduces (1) a client-specific prototype loss as an auxiliary training objective; and (2) an aggregation algorithm guided by such prototypes across clients. Our experimental results in an extensive array of non-i.i.d.~settings show that \emph{Calibre} achieves state-of-the-art performance in terms of both mean accuracy and fairness across clients. Code repo: \href{https://github.com/TL-System/plato/tree/main/examples/ssl/calibre}{calibre}.

\end{abstract}


\input{intro}

\input{related}
\input{prelim}

\input{analysis}

\input{calibre}

\input{exps}

\input{exps_detail}

\section{Concluding Remarks}
\label{sec:concl}

In this paper, we focused on personalized federated learning and have thoroughly investigated how a fair model performance across clients can be achieved while maximizing the average training performance. Our objective for designing \emph{Calibre}, our new personalized federated learning framework, was to adopt self-supervised learning (SSL) in the training stage to train a global model that could generalize well to individual clients. However, we empirically found that although this model benefits model fairness across clients, its average accuracy is even worse. It turned out that the root cause for this observation was the lack of class separation information in the representations extracted by the global model. After a comprehensive theoretical analysis, we presented a theorem, called \emph{generality-personalization tradeoff}, to include cluster information in the representation of SSL. With these insights, we proposed a new contrastive prototype adaptation mechanism that was able to improve the mean accuracy while maintaining a uniform accuracy across clients (model fairness). We showed convincing results from a wide array of experiments that \emph{Calibre} achieves higher model fairness, maintains better mean accuracy, and shows more consistent performance on multiple non-i.i.d.~data scenarios than its state-of-the-art alternatives in the literature. 

\bibliographystyle{IEEEtran}
\balance
\bibliography{IEEEabrv,main}

\end{document}

%% file: intro.tex
\section{Introduction}

With federated learning (FL)~\cite{fedavg-aistats17}, multiple clients collaboratively train a global model while keeping their local datasets private. However, with data heterogeneity, the local accuracies that clients achieve may diverge significantly after training the same model on their different local datasets. This leads to the failure of achieving a more uniform --- or \emph{fair} --- distribution of client test accuracies \cite{qffl-iclr20}. Known as \textit{model unfairness}, this challenge motivated existing research on \textit{personalized federated learning}~\cite{tpfl-tnnls22}, which focused on training a global model for each client to use as its starting point when training its personalized model. Such personalized models are better aligned with local data and may therefore improve fairness.

Unfortunately, when the local data distributions across clients are severely non-independent and identically distributed (non-i.i.d.), it remains challenging to improve model fairness across clients \emph{while maintaining} high overall performance. For example, while fine-tuning a trained global model within each client can improve the mean client test accuracy, it often leads to a higher variance, resulting in unfairness. This discrepancy primarily stems from significant variations in the local data distribution of individual clients, when compared to the global data distribution across all clients. In the context of data heterogeneity with non-i.i.d.~data across clients, it would be ideal to achieve both low variance to improve fairness in local test accuracies across clients, while simultaneously maximizing the overall mean accuracy.

To achieve both objectives, it has been shown in the recent literature that personalized model training needs high-quality representations of samples as its starting point. Therefore, existing research \cite{lgfedavg-nips19, fedmvt-ijcai20, fedrep-icml21, fedbabu-iclr22} attempted to train an encoder as the global model, which is capable of capturing generic representations from the underlying non-i.i.d.~data. However, training models with these existing mechanisms depend on label information. Consequently, when some classes or labels are over-represented in the data from certain clients, model training may become biased towards these majority classes. This can result in limited representations that are hard to generalize well across all clients, exacerbating the issue of unfairness. Furthermore, the overall performance of these existing mechanisms is also significantly reduced if clients do not have an adequate number of labeled samples.

In this paper, we argue that utilizing self-supervised learning (SSL) as a training approach for the global model is an effective solution to these issues. SSL allows the global model to be trained in an unsupervised manner without the need for labels, thereby mitigating the issues related to label skewness. More importantly, with the objective function of SSL, the global model is trained to extract invariant features across clients. With transferable representations obtained from SSL training as a starting point, each client can then train a high-quality personalized model, even if the number of samples is limited. Although recent efforts in the literature, such as FedEMA~\cite{fedema-iclr22}, explored the usage of SSL in conventional federated learning, how SSL can be used for \emph{personalized} federated learning remains uncharted territory.

In this paper, we propose a new framework, referred to as \emph{Calibre}, that employs self-supervised learning to train global models with a focus on two fundamental objectives of personalization: the best possible \emph{fairness} across all clients, and optimal mean client test accuracies. \emph{Calibre} contains two federated learning stages: the \emph{training} stage that trains a global model using SSL, and the \emph{personalization} stage that allows clients to utilize the global model as the feature extractor to train personalized models.

From empirical experiments, we observed that even though training a global model with SSL contributes to fairness to some extent, the quality of the produced personalized model is poor. More specifically, the test accuracy of this personalized model is even lower than the accuracy of a local model trained without relying on the global model. Furthermore, within our proposed framework that facilitates fair comparisons, we conducted experiments involving other recent works \cite{fedema-iclr22, fedu-iccv21, fedca-edgesys20} under non-i.i.d.~data. Surprisingly, none of these approaches achieved competitive performance in personalized models. To tackle this issue, we conducted a qualitative analysis of SSL representations to gain better insights. First, we observed that the learned SSL representations collected from different clients were mixed without presenting meaningful and discriminative information. For instance, no distinct clusters emerged in the representations, even when they originated from the same class. Second, within each client, SSL representations exhibit unclear class boundaries, resulting in poorer class separation, which is essential for subsequent personalized model learning. 

We naturally wonder: how can SSL representations be calibrated to improve personalization while preserving their fairness guarantee?
Towards answering this question, we attempt to build theoretical insights from an information theoretic perspective \cite{metamo-iclr20}, with the objective of exposing how personalized learning depends on the global model and the local dataset. On the one hand, once the global model is trained to solely capture transferable SSL representations, it will have a limited ability to extract client-specific information, such as clusters of the local dataset, necessary for personalization. On the other hand, a global model unable to capture generic representations will contain sparse information from the datasets of whole clients, making only information from the local dataset contribute to personalized learning. To obtain balanced information usage during personalization, we formulate this process as an optimization problem, and derive important theoretical results to analyze the \emph{generality-personalization tradeoff}.

Taking advantage of our theoretical insights, the core contribution of this paper is to pursue fair and accurate personalized FL by calibrating the SSL representations with a new contrastive prototype adaptation mechanism. During the local update of each client, client-specific prototype regularizers can work seamlessly with any SSL approach to optimize the global model toward capturing generic and clustered representations. Each client then computes the average distance between its samples and their corresponding prototypes. Such average distance can be effectively used to measure the local divergence rate, which acts as a weighting factor during the server aggregation. Through an extensive array of experiments conducted across various non-i.i.d.~settings using the \texttt{CIFAR-10}, \texttt{CIFAR-100}, \texttt{STL-10} \cite{stl10-aistats11} datasets, we illustrate that the utilization of a lightweight personalized model, specifically a linear classifier, would be sufficient for \emph{Calibre} to achieve state-of-the-art performance in terms of both mean accuracy and fairness. In addition, \emph{Calibre} also generalizes well to unseen clients that have not participated in the training process.

%% file: related.tex
\section{Related Work}
\label{sec:related}

Existing research in the FL literature has shown that the quality of the global model deteriorates when clients across the board have non-i.i.d. data. \textit{Personalized federated learning} (pFL) \cite{tpfl-tnnls22, fedamp-aaai21} was proposed with the target of training personalized models for individual clients while maintaining fairness across clients, in that the variance of local test accuracies is low. One of the primary research directions \cite{lgfedavg-nips19, fedrep-icml21, fedbabu-iclr22, fedper-arxiv19} exploits representation learning to train a global model capable of extracting transferable representations. With this model as the starting point, each client can train a high-quality personalized model. Particularly, FedRep \cite{fedrep-icml21} jointly learns a single global representation and many local heads. FedBABU \cite{fedbabu-iclr22} shares a similar two-stage training approach to our work. In the first stage, it trains a shared encoder using decentralized datasets, after which each client optimizes its local head by leveraging the features extracted from the fixed encoder. In contrast to approaches that heavily rely on strong supervision information, such as labels, our work takes a different approach by incorporating self-supervised learning (SSL) \cite{revisitssl-CVPR19} into personalized FL. By leveraging SSL, we enhance the generality of the global model without the need of using labels in training data samples.

Exploring unsupervised training methods for addressing non-i.i.d.~challenges in federated learning is a burgeoning research area. Existing studies that were closely related to our work, such as \cite{fedema-iclr22, fedu-iccv21, fedca-edgesys20}, employed SSL frameworks, including BYOL \cite{byol-nips20}, SimCLR \cite{simclr-icml20}, and Simsiam \cite{simsiam-cvpr21}. The overarching objective was to optimize the global model using multiple augmentations or views of the same input data. Specifically, FedEMA \cite{fedema-iclr22} conducted a comprehensive empirical investigation into federated self-supervised learning. Based on this study, FedEMA introduced a novel approach that combined elements of BYOL and employed an exponential moving average (EMA) scheme. To the best of our knowledge, no prior research explored the effectiveness of self-supervised learning in the context of personalized FL, while simultaneously considering fairness and overall performance. 


%% file: prelim.tex
\section{Problem Formulation}
\label{sec:probstate}

Unlike conventional federated learning that trains a global model $\bm{\theta}$ across $C$ clients, personalized federated learning (pFL) aims to produce personalized models $\left\{\bm{\phi}^c\right\}_{c=1}^C$ for individual clients. This leads to an optimization problem given by $\min_{\left\{\bm{\phi}^c\right\}_{c=1}^C} \frac{1}{C} \sum_c L^c \left(\bm{\phi}^c; D^{c}\right)$, where $L^c$ and $D^{c}$ is the loss function and the local train set of client $c$, respectively. Among various methodologies toward achieving personalized FL, this paper focuses on the paradigm \cite{fedrep-icml21,fedbabu-iclr22,fedema-iclr22} in which clients cooperatively train a global model $\bm{\theta}$, containing fully convolutional layers $\bm{\theta}_b$ and fully-connected layers $\bm{\theta}_h$. After reaching convergence, each client utilizes $\bm{\theta}_b$ to extract features to train its $\bm{\phi}$ based on the local dataset. For the purpose of evaluation, each client tests its trained $\bm{\phi}^c$ on the local test dataset $D^{\prime c}$. In the context of image classification that we focus on, clients compute accuracies $\left\{a^1,...,a^C\right\}$. This paper focuses on a common challenge in non-i.i.d.~data scenarios, particularly where the label distributions vary considerably across clients.

\subsection{Model Fairness and Overall Performance}
\label[type]{sec: fairness}

Mean accuracy computed on $\left\{a^1,...,a^C\right\}$ presents the overall performance but fails to give insights into how well individual clients can train personalized models. Under the non-i.i.d.~data of clients, the global model $\bm{\theta}$ may be trained to be biased towards the data distribution of certain clients. Therefore, when $\bm{\theta}$ is eventually used by clients for personalization under non-i.i.d.~data, such bias introduces highly variable performance between different clients as the trained global model cannot generalize well to some local datasets.

This disparity in model performance across different clients is commonly recognized as model unfairness \cite{ditto-icml21}. We formally extend this to personalized FL. Specifically, \textit{fairness} is defined as the case if, based on the trained global model, clients can generate personalized models with similar performance. In terms of accuracy, this leads to $\left\{a^1, \ldots,a^C\right\}$ of clients presenting a low divergence, meaning that the variance computed from these test accuracies is low. 

We argue that targeting a high overall performance but sacrificing model fairness, or vice versa, is extremely detrimental when applying personalized FL to real-world applications. Therefore, in this paper, our objective is to impose better fairness by decreasing the accuracy variance, while still achieving the best possible mean accuracy.

\subsection{Personalized FL with Self-Supervised Learning}

As pointed out by the existing literature \cite{fedrep-icml21,fedbabu-iclr22}, the core idea to improve fairness is to train the global model toward extracting the generic features containing common patterns of the dataset. And thus, any clients can train personalized models based on these extracted features of local samples. Motivated by these observations, we propose introducing self-supervised learning (SSL), an unsupervised learning approach, to train the global model under personalized FL, pFL-SSL. This is because the objective of SSL does not include labels, and as a result, learned representation is not tied to specific labeled outcomes. This generalizes well to local samples with different label distributions; with non-i.i.d.~data, the trained global model will be less biased towards the data distribution of a certain subset of clients.

Therefore, the preliminary design is to train the global model with SSL and then perform personalization on each individual client based on the transferable representation learned by the trained global model. First, during the \emph{training stage}, the global model $\bm{\theta}$ is trained with the loss function of SSL \cite{simclr-icml20} till reaching the convergence. Subsequently, during the \emph{personalization stage}, each client utilizes $\bm{\theta}_b$ of the $\bm{\theta}$ as the feature extractor of the local samples and then $\bm{\phi}^c$ is trained in a supervised manner with the cross entropy loss. As the $\bm{\theta}_b$ is trained to extract generic representations, the $\bm{\phi}^c$ is designed to be lightweight, such as a linear classifier.

Such a pFL-SSL design ensures compatibility with a wide range of state-of-the-art SSL methods. One only needs to change the SSL method in the training stage to obtain a new approach. For example, one can directly implement pFL-BYOL, pFL-SimCLR, pFL-SimSiam, and pFL-MoCoV2 by introducing BYOL \cite{byol-nips20}, SimCLR \cite{simclr-icml20}, SimSiam \cite{simsiam-cvpr21}, and MoCoV2 \cite{moco-cvpr20}, respectively.

\subsection{Representations with Fuzzy Class Boundaries}
\label[type]{prob: fuzzy}

\begin{figure}[t]
    \centering
    \includegraphics[width=0.19\textwidth]{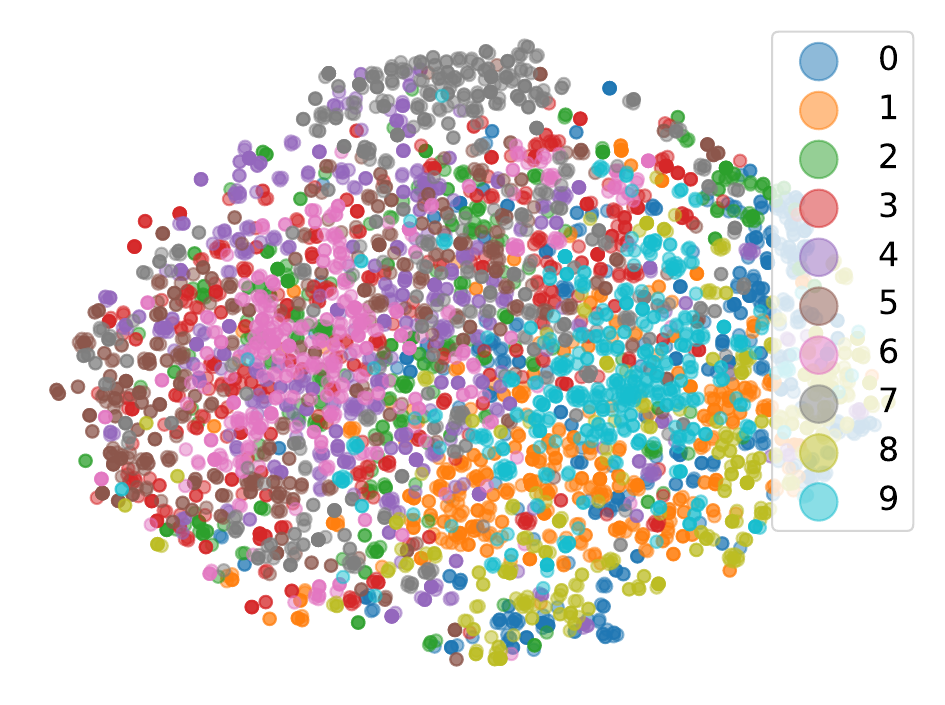}
    \includegraphics[width=0.19\textwidth]{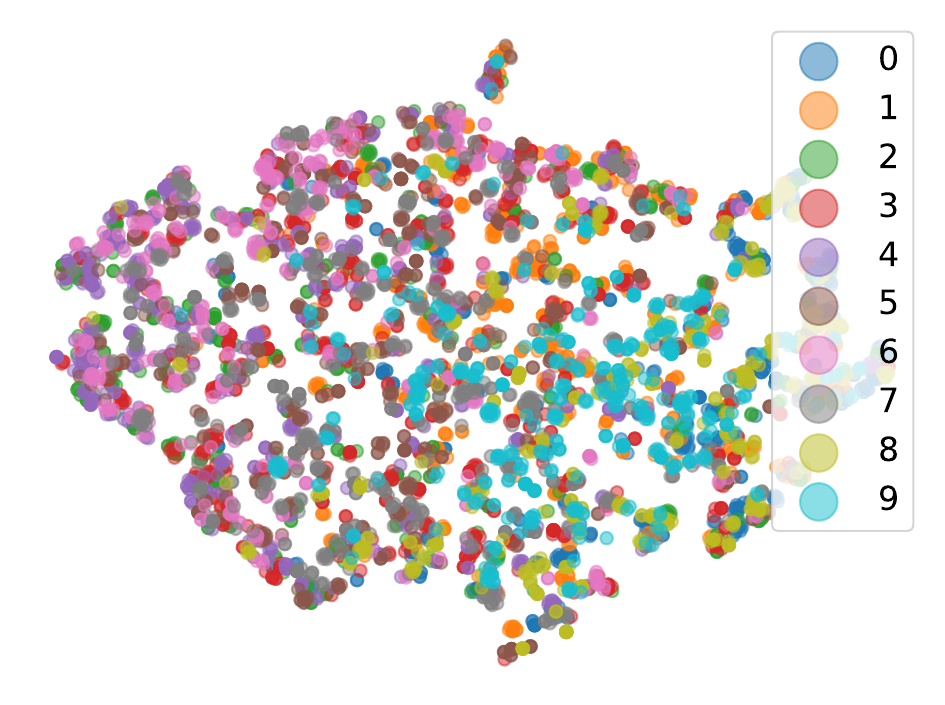}
    \caption{Illustrations of 2D t-SNE embeddings of SSL representations learned by four methods. Moving from left to right, the representations are derived from local samples of 10 out of 100 clients, using encoders trained by pFL-SimCLR and pFL-BYOL respectively.}
    \label{fig: fuzzy0}
    \vspace{-0.2cm}
\end{figure} 

\begin{figure}[t]
    \centering
    \includegraphics[width=0.24\textwidth]{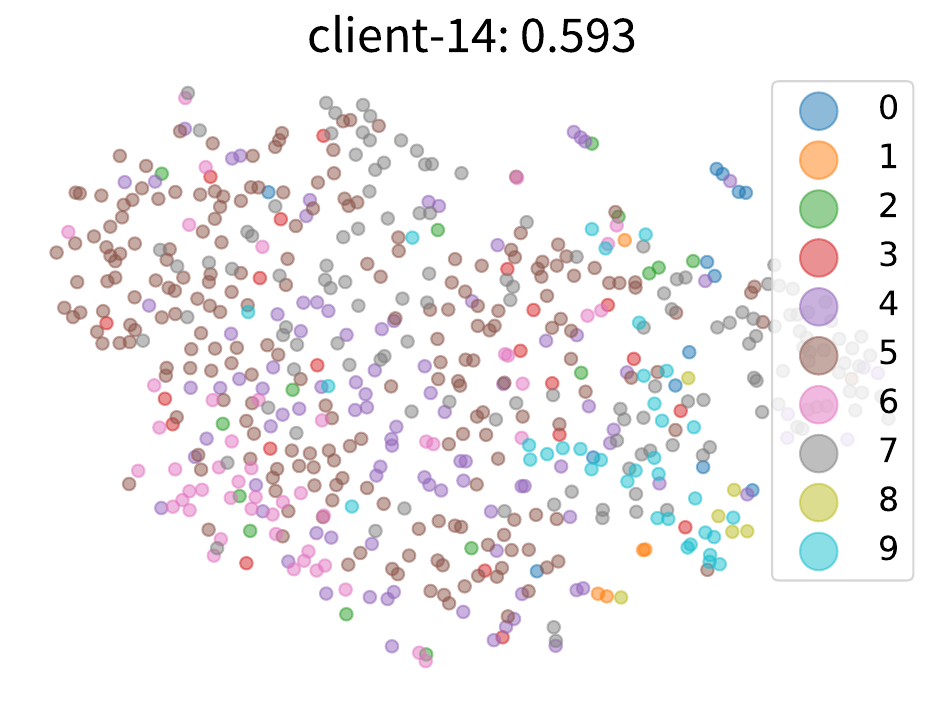}
    \includegraphics[width=0.24\textwidth]{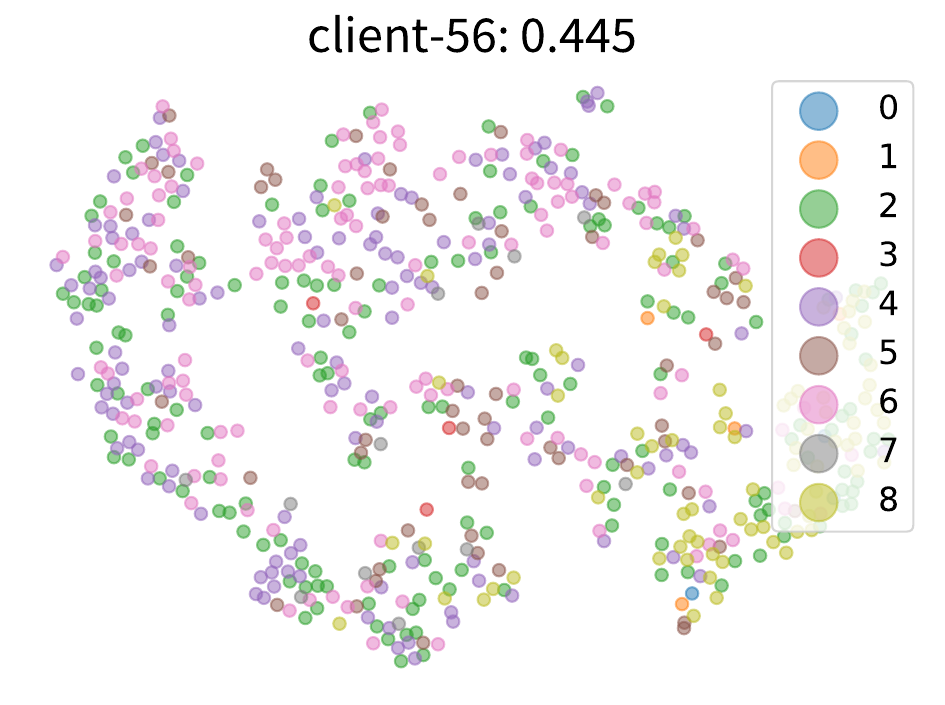}
    \caption{Visualization of 2D t-SNE embeddings of client representations and test accuracy derived from pFL-SimCLR and pFL-BYOL methods. The first three representations are from pFL-SimCLR, while the last three are from pFL-BYOL. These examples are randomly selected from a pool of 100 clients.}
    \label{fig: fuzzy1}
    \vspace{-0.5cm}
\end{figure} 

With pFL-SSL, the transferable representations extracted by $\bm{\theta}$ on local samples are expected to demonstrate two key properties: 1) exhibiting distinct class boundaries, even in the presence of class-imbalanced samples for each client \cite{rethinkssl-nips20}, and 2) encompassing transferable generic semantics. However, based on our experiment performed on \texttt{CIFAR-10} under non-i.i.d.~data across $100$ clients, we observed that the representations learned by pFL-SimCLR and pFL-BYOL, as shown in \cref{fig: fuzzy0} and \cref{fig: fuzzy1}, do not exhibit these properties. In contrast, the 2D embeddings of features reveal the following two limitations of PFL-SSL. 

\noindent \textbf{Fuzzy cluster boundaries across clients}. \cref{fig: fuzzy0} shows the distribution of representations collected from multiple clients. It is obvious that representations of pFL-SimCLR and pFL-BYOL cannot be clustered into distinct groups. This means that these approaches only learn the generic semantics without holding class information. Specifically, the objective of pFL-SimCLR and pFL-BYOL is designed to measure the agreement or discrepancy between different views of the same data. This leads to representations with higher generality, but no cluster information is included. Without such information, the cluster boundaries of representations are fuzzy, and thus, it is hard to perform classification. 

\noindent \textbf{Fuzzy cluster boundaries within each client}. As shown in \cref{fig: fuzzy1}, even within individual clients, representations of pFL-SimCLR and pFL-BYOL fail to shape clear class boundaries and thus cannot embrace better class separation for subsequent personalized learning. In the non-i.i.d.~data, individual clients exhibit class-imbalanced samples, meaning that certain classes have a considerably larger number of samples than others, as depicted in \cref{fig: fuzzy1}. Previous work \cite{rethinkssl-nips20} has highlighted that SSL representations form distinct clusters for class-imbalanced samples under central learning. However, the observations from pFL-SimCLR and pFL-BYOL demonstrate that SSL representations fail to maintain such a benefit under the personalized FL with non-i.i.d. Specifically, the representations do not exhibit clear boundaries that separate samples from different classes.

Hence, we argue that pFL-SSL enables training a global model that captures transferable representations holding generic patterns. Consequently, the personalized learning of individual clients can benefit equally from these representations. This guarantees the model fairness. However, as cluster information and clear boundaries are not contained in SSL representations, clients are unable to train high-quality personalized models. As a result, the overall performance tends to be low. 

Finally, our question in this paper is how we should include cluster information when training the global model with SSL. Thus, SSL representations are calibrated to achieve accurate personalized FL while maintaining its capability in fairness guarantees.

%% file: analysis.tex
\section{\emph{Calibre}: a New pFL Framework}
\label{sec:calibre}

In this section, we begin with a theoretical exploration of the impact of self-supervised representations containing fuzzy class boundaries on personalization. As a result, the derived theorem guides our design of \emph{Calibre} towards fair and accurate personalized FL. 

To train the personalized model $\bm{\phi}$, the $\bm{\theta}_b$ extracts features from the input $\bm{x}^{\prime}$, resulting in the latent representation $\overline{\bm{z}}=f_{\bm{\theta}_b}\left(\bm{x}^{\prime}\right)$, which is utilized to make a prediction with $\bm{\phi}$, denoted as $\overline{y}^{\prime}$. Therefore, the predicted label $\overline{y}^{\prime}$ is obtained based on information derived from: 1) the local dataset $D$, where $\bm{x}^{\prime} \in D$ is processed by models to produce $\overline{y}^{\prime}$; and 2) $\bm{\theta}_b$, the encoder responsible for extracting latent features, directly contributing to $\overline{y}^{\prime}$.
    
\subsection{Generality-Personalization Tradeoff Theorem}
The learning for the personalized model $\bm{\phi}$ of each client depends on the self-supervised representations $\overline{\bm{z}}$ extracted by the global $\bm{\theta}_b$. We denote this process as $\bm{x}^{\prime}\rightarrow \overline{\bm{z}} \rightarrow \overline{y}^{\prime}$. Thus, from the information theory perspective, this process can be formulated by two information flows, including $\bm{x}^{\prime} \in D \rightarrow \bm{\phi} \rightarrow \overline{y}^{\prime}$ and $\bm{\theta}_b \rightarrow \overline{y}^{\prime}$.
 
\textbf{Excessively generic representations.} On the one hand, when self-supervised representations $\bm{z}$ primarily exhibit generic features across clients, such as sample-wise similarity, the training for $\bm{\phi}$ is adversely affected as a result of failing to capture the class distinguishability information from client-specific data $D$, as witnessed in \cref{fig: fuzzy1}. This can be denoted as $q(\overline{y}^{\prime}|\bm{x}^{\prime},\bm{\phi})q(\bm{\phi}|D,\bm{\theta}_b)=q(\overline{y}^{\prime}|\bm{x}^{\prime},\bm{\phi})q(\bm{\phi}|\bm{\theta}_b)$, in which the personalized model $\bm{\phi}$ is independent of the dataset $D$ in the client, such that $q(\overline{y}^{\prime}|\bm{x}^{\prime},\bm{\theta}_b,D)=q(\overline{y}^{\prime}|\bm{x}^{\prime},\bm{\theta}_b)$. Hence, we have Lemma \ref{bounded-yd} which improves personalization by maximizing the information from $D$ to the prediction.

\begin{lemma}
\label{bounded-yd}
Given $\bm{\theta}_b$ and the learned representation $\overline{\bm{z}}$, we can implicitly encourage the model to use the client's local training data $D$ by maximizing the lower bound of $I(\overline{y}^{\prime};D|\bm{\theta}_b,\overline{\bm{z}})$ shown as $I(\bm{x}^{\prime};\overline{y}^{\prime}|\bm{\theta}_b) - E\left[KL(p(\overline{z}|\bm{x}^{\prime}, \bm{\theta}_b)||r(\overline{z}))\right]$, where $r(\overline{z})$ is a variational approximation to the marginal.
\end{lemma}

\begin{proof}
    With $KL(p(\overline{y}^{\prime}|\overline{\bm{z}}, \bm{\theta}_b)|q(\overline{y}^{\prime}|\overline{\bm{z}}, \bm{\theta}_b))\geq 0$, $I(\overline{y}^{\prime};\overline{\bm{z}} |\bm{\theta}_b)$ has a lower bound $\int p(\overline{\bm{z}}, \overline{y}^{\prime}|\bm{\theta}_b)\log{q(\overline{y}^{\prime}|\overline{\bm{z}}, \bm{\theta}_b)}d\overline{y}^{\prime}d\overline{\bm{z}}$ , where $q(\overline{y}^{\prime}|\overline{\bm{z}}, \bm{\theta}_b)$ is a variational approximation to the target distribution $p(\overline{y}^{\prime}|\overline{\bm{z}}, \bm{\theta}_b)$. 
Then, given 
\begin{equation}
    \int p(\overline{\bm{z}}, \overline{y}^{\prime}|\bm{\theta}_b)d\overline{y}^{\prime}d\overline{\bm{z}}=\int p(\bm{x}^{\prime}, \overline{y}^{\prime}| \theta)p(\bm{z}|\bm{x}^{\prime}, \bm{\theta}_b)d\bm{x}^{\prime}d\overline{y}^{\prime}d\overline{\bm{z}}
\end{equation}

we can approximate the lower bound of $I(\overline{y}^{\prime};\overline{\bm{z}}, |\bm{\theta}_b)$ by using the empirical data distribution of $p(\bm{x}^{\prime}, \overline{y}^{\prime})$ on $D^{\prime}$. We obtain $I(\overline{y}^{\prime};\overline{\bm{z}}, |\bm{\theta}_b) \geq E\int p(\overline{\bm{z}}|\bm{x}^{\prime}_n, \bm{\theta}_b)\log q(y_n|\overline{\bm{z}}, \bm{\theta}_b) d\overline{\bm{z}}$. Utilizing the reparameterization trick, this can be estimated as $E_{\bm{x}^{\prime}} E_{\epsilon \sim N(0, I)}\left[\log q(\overline{y}^{\prime}|\bm{x}^{\prime}, \bm{\theta}_b, \epsilon)\right]$.
\end{proof}

\textbf{Excessive personalization.} On the other hand, personalized learning is hindered when the trained global encoder $\bm{\theta}_b$ mainly concentrates on capturing client-specific class distribution from local samples, while providing limited prior generic information learned from datasets of other clients. As a result, with no information from $\bm{\theta}_b$ being included, the training for $\bm{\phi}$ relies solely on the information flow from the $D$. This can be mathematically denoted as $q(\overline{y}^{\prime}|\bm{x}^{\prime},\bm{\phi},\bm{\theta}_b)q(\bm{\phi}|D,\bm{\theta}_b)=q(\overline{y}^{\prime}|\bm{x}^{\prime},\bm{\phi})q(\bm{\phi}|D)$, in which $\bm{\phi}$ only depends on limited local samples $D$, thus significantly damaging the training speed and the prediction performance. Therefore, we should enhance the generic of self-supervised representations by improving the mutual information $I(\overline{y}^{\prime};\overline{\bm{z}} |\bm{\theta}_b)$. Thus, we have lemma \ref{bounded-yz}.

\begin{lemma}
\label{bounded-yz}
With the representation $\overline{\bm{z}}$ obtained by applying the $\bm{\theta}_b$ on the sample $\bm{x}^{\prime}$, more information of $\bm{\theta}_b$ can be given to personalized prediction $\overline{y}^{\prime}$ by maximizing the mutual information $I(\overline{y}^{\prime};\overline{\bm{z}} |\bm{\theta}_b)$ with a lower bound estimated as $E_{\bm{x}^{\prime}, y^{\prime}}\int p(\overline{\bm{z}}|\bm{x}^{\prime}, \bm{\theta}_b)\log p(\overline{y}^{\prime}|\overline{\bm{z}}, \bm{\theta}_b) d\overline{z}$.
\end{lemma}

\begin{proof}
We first get a zero term $I(\bm{x}^{\prime};\overline{y}^{\prime}|D, \bm{\theta}_b,\overline{\bm{z}})$. Then, we can naturally follow the theoretical analysis in the Appendix A.2 of the work \cite{metamo-iclr20} to obtain $I(\overline{y}^{\prime};D|\bm{\theta}_b,\overline{\bm{z}}) \geq I(\bm{x}^{\prime};\overline{y}^{\prime}|\bm{\theta}_b,\overline{\bm{z}})$. For each client, the input $\bm{x}^{\prime}$ used to predict $\overline{y}^{\prime}$ is the sample of the test set that has a consistent class distribution with the training set. Then, we have 
\begin{equation}
\begin{aligned}
I(\overline{y}^{\prime};D|\bm{\theta}_b,\overline{\bm{z}}) &\geq I(\bm{x}^{\prime};\overline{y}^{\prime}|\bm{\theta}_b,\overline{\bm{z}}) \geq I(\bm{x}^{\prime};\overline{y}^{\prime}|\bm{\theta}_b)-I(\bm{x}^{\prime};\overline{\bm{z}}|\bm{\theta}_b)\\
&\geq I(\bm{x}^{\prime};\overline{y}^{\prime}|\bm{\theta}_b) - E\left[KL(p(\overline{\bm{z}}|\bm{x}^{\prime}, \bm{\theta}_b)||r(\overline{\bm{z}}))\right]
\end{aligned}
\end{equation}
where the second inequality is obtained by using the chain rule of mutual information and then removing the mutual information term $I\left(\bm{x}^{\prime}, \bm{\theta}_b|\bm{\theta}_b, \overline{y}^{\prime}\right)$. Then, with the variational approximation $r(\overline{\bm{z}})$, we can achieve the final term by extending the third term to $I(\bm{x}^{\prime};\overline{y}^{\prime}|\bm{\theta}_b) - E_{p(\bm{x}^{\prime})p(\overline{\bm{z}}|\bm{x}^{\prime},\bm{\theta}_b)}\left[log\frac{p(\overline{\bm{z}}|\bm{x}^{\prime}, \bm{\theta}_b)}{p(\overline{\bm{z}}|\bm{\theta}_b)}\right]$ which can be bounded by $E\left[log\frac{p(\overline{\bm{z}}|\bm{x}^{\prime}, \bm{\theta}_b)}{r(\overline{\bm{z}})}\right]$ in which $p(\overline{\bm{z}}|\bm{x}^{\prime}, \bm{\theta}_b)$ can be computed as a  deterministic function by using the reparameterization trick. However, we do not compute this term in our work for simplicity.
\end{proof}

Considering both conclusions, fair and accurate personalized learning can be achieved by calibrating the self-supervised representations to capture client-specific concepts while preserving high-level semantics. This gives rise to a mutual information maximization problem, as depicted below:

\begin{equation}
\begin{aligned}
&\max I(\overline{y}^{\prime}; \overline{\bm{z}^{\prime}}|\bm{\theta}_b, D) \\
&s.t. I(\overline{y}^{\prime}, D; \bm{\theta}_b|\bm{x}^{\prime}) \leq I_c
\end{aligned}
\label{eq: main_target_eq}
\end{equation}
where $I_c$ is the constant information constraint.

Addressing this optimization problem leads to the core theorem that guides the methodology design of this paper. 

\begin{theorem}[Generality-Personalization Tradeoff]
\label{theo: to_fa}
To calibrate the self-supervised representations for personalization in federated learning, the global encoder $\bm{\theta}_b$ is trained to attain balanced information flows as formulated in Eq.~(\ref*{eq: main_target_eq}). This is equivalent to focusing on the following objectives during the training process:
\begin{equation}
\begin{aligned}
\max_{\bm{\theta}_b} &E\int p(\overline{\bm{z}}|\bm{x}^{\prime}, \bm{\theta}_b)\log p(\overline{y}^{\prime}|\overline{\bm{z}}, \bm{\theta}_b) d\overline{z} 
\\ &- E_{\bm{\theta}_b|D}E_{D^{\prime c}}\left[l^c(\bm{\theta}_b; \bm{x}^{\prime}, y^{\prime}) \right]  \\
&-\beta E\left[{KL}(p(\bm{\theta}_b|D, \bm{x}^{\prime}, \overline{y}^{\prime})||r(\bm{\theta}_b)\right] 
\\ &- E\left[KL(p(\overline{z}|\bm{x}^{\prime}, \bm{\theta}_b)||r(\overline{z}))\right]
\end{aligned}
\label{eq: theorem}
\end{equation}
where $\bm{\phi}^c \sim q(\bm{\phi}^c|D^c,\bm{\theta}_b)$, $(\bm{x}^{\prime},y^{\prime}) \sim D^{\prime c}$. $l^c$ can a classification loss function.
\end{theorem}

\begin{proof}
    Given the optimization problem:
    \begin{equation}
    \begin{aligned}
    &\max I(\overline{y}^{\prime}; \overline{\bm{z}^{\prime}}|\bm{\theta}_b, D) \\
    &s.t. I(\overline{y}^{\prime}, D; \bm{\theta}_b|\bm{x}^{\prime}) \leq I_c
    \end{aligned}
    \label{eq: proofmain_target_eq}
    \end{equation}
    where $I_c$ is the constant information constraint.
    
    The predictive process can be defined as a Markov chain shown as $Y \leftrightarrow X \leftrightarrow Z$, where Y is the label, X is the input sample and Z is the latent feature. The optimization problem in Eq.\ref{eq: proofmain_target_eq} is solved by introducing a Lagrange multiplier $\beta$. 
    
    Then, the problem is transformed into maximizing the formula below 
    \begin{equation}
    \begin{aligned}
    \triangle=&I(\overline{y}^{\prime}; \overline{\bm{z}}, |\bm{\theta}_b, D) - \beta I(\overline{y}^{\prime}, D; \bm{\theta}_b|\bm{x}^{\prime}) \\
    =&I(\overline{y}^{\prime};\overline{\bm{z}}, |\bm{\theta}_b)+I(\overline{y}^{\prime};D|\bm{\theta}_b,\overline{\bm{z}} )-I(\overline{y}^{\prime};D|\bm{\theta}_b)\\
    & - \beta I(\overline{y}^{\prime}, D; \bm{\theta}_b|\bm{x}^{\prime})
    \end{aligned}
    \label{eq: loew_bound}
    \end{equation}
    
    Then, the final term equals to $E\left[\log \frac{p(\overline{y}^{\prime},D,\bm{\theta}_b|\bm{x}^{\prime})}{p(\overline{y}^{\prime},D|\bm{x}^{\prime})p(\bm{\theta}_b|\bm{x}^{\prime})} \right]=E\left[\log \frac{p(\bm{\theta}_b|D, \bm{x}^{\prime}, \overline{y}^{\prime})}{p(\bm{\theta}_b|\bm{x}^{\prime})} \right]$, which has the lower bound $E\left[{KL}(p(\bm{\theta}_b|D, \bm{x}^{\prime}, \overline{y}^{\prime})||r(\bm{\theta}_b)\right]$ where $r(\bm{\theta}_b)\sim N(0,I)$ can be a variational approximation to the target distribution of $\bm{\theta}_b$.
    
    After introducing the lemmas 1, lemma 2 and the above obtained lower bound for the $\beta$ term, the chain rule flatten of $\triangle$ additional includes the mutual information term $I(\bm{x}^{\prime};\overline{y}^{\prime}|\bm{\theta}_b) - I(\overline{y}^{\prime};D|\bm{\theta}_b)$. If the $I(\bm{x}^{\prime};\overline{y}^{\prime}|\bm{\theta}_b)$ is 0, the model predictions do not depend on the given $\bm{x}^{\prime}$, leading to low accuracy. Thus, maximizing $I(\bm{x}^{\prime};\overline{y}^{\prime}|\bm{\theta}_b)$ term is equivalent to minimizing classification losses $l^c$ for high accuracy. Similarly, as the online encoder $\bm{\theta}_b$ learns the generic representation for prediction on the client, the representation containing distinguishable features for classification leads to predictions that are less dependent on local training data sets, i.e., a smaller $I(\overline{y}^{\prime};D|\bm{\theta}_b)$.  Therefore, the local representation learned by the $\bm{\theta}_b$ in each client is expected to contain prototype information for the local classes for direct prediction on the $D^{\prime}$. Therefore, we can replace this term with a conditional classification loss objective function that is computed conditioned on the global model $\bm{\theta}_b$.
\end{proof}

%% file: calibre.tex
\subsection{Design}
\begin{algorithm}[h]
    \SetKwInOut{Input}{Input}
    \SetKwInOut{Output}{Output}
    \SetAlgoLined
    
        \Input{Global model $\bm{\theta}$. One batch of samples $B$.}
        \KwOut{Computed loss $L$}
    
        \For{$\bm{x}_i \in B$, $i \in \left\{1,...,N\right\}$}{
            Use augmentation functions $A$ of SimCLR \cite{simclr-icml20} \\
            $\widehat{\bm{x}}_{2i-1} = A\left(\bm{x}_i\right)$; $\widehat{\bm{x}}_{2i} = A\left(\bm{x}_i\right)$\\
            $\bm{z}_{2i-1} = f_{\bm{\theta}_b}\left(\widehat{\bm{x}}_{2i-1}\right)$; $\bm{z}_{2i} = f_{\bm{\theta}_b}\left(\widehat{\bm{x}}_{2i}\right)$\\
            $\bm{h}_{2i-1}=f_{\bm{\theta}_h}\left(\bm{z}_{2i-1}\right)$; $\bm{h}_{2i}=f_{\bm{\theta}_h}\left(\bm{z}_{2i}\right)$ \\
           }
        $l_s = \textrm{NTXent}\left(\left\{\bm{h}_i\right\}^{2N}_{i=1}\right)$\\
        \For{$k \in K_r$}{
            $\bm{\nu}_{2k-1}=\frac{1}{N_k}\sum_{j \in I^k_e}\bm{h}_{j}$\\
            $\bm{\nu}_{2k}=\frac{1}{N_k}\sum_{j \in I^k_o}\bm{h}_{j}$

        }
        \PyCode{$l_p = \textrm{NTXent}\left(\left\{\bm{\nu}_i\right\}^{2K}_{i=1}\right)$}\\
        $K_r = KMeans\left(\bm{z}\right)$, $\bm{z}=\left[z_{2i-1}, z_{2i}\right]$ \\
        \For{$k \in K_r$}{
            $\bm{v}_{k}=\frac{1}{N_k}\sum_{j \in I^k_o}\bm{z}_{j}$\\
        }
        $L_n = \sum_{k \in K^r} \frac{-1}{ N_k }\sum_{j \in I^k_e}\log \frac{\exp\left(\bm{z}_j \cdot \bm{v}_k / \tau \right)}{\sum_{a \in I_{e}\diagdown k }\exp\left(\bm{z}_a \cdot \bm{v}_k / \tau \right)}$\\    

        $L^c = l_c + l_s + \alpha \left(l_p+ l_n\right)$\\
    \caption{Loss computation for \emph{Calibre} (SimCLR)}
    \label{algo:Calibre}
    \end{algorithm}

Building on the Theorem \ref{theo: to_fa}, the core of \emph{Calibre} to fair and accurate personalized FL is a novel loss function $L^c$ containing four terms as shown in \cref{eq: theorem}. 

\noindent \textbf{Prototype generation.} Extending the concept of prototypes \cite{fedproto-aaai22}, \emph{Calibre} first generates pseudo labels through a straightforward clustering algorithm, such as KMeans, thereby the prototype vector for the $k$-th cluster is calculated as the average of encodings assigned to this group. 

\noindent\textbf{Client-adaptive prototype regularizers.} The last two terms of the \cref{eq: theorem} can be safely omitted without negatively affecting local learning. Specifically, $\left[{KL}\left(p(\bm{\theta}_b|D, \bm{x}^{\prime}, \overline{\bm{y}}^{\prime})||r\left(\bm{\theta}_b\right)\right)\right]$ is removed because the $\bm{\theta}_b$ is designed upon discriminative parameters. Kullback-Leibler (KL) divergence term $E\left[KL(p(\overline{z}|\bm{x}^{\prime}, \bm{\theta}_b)||r(\overline{z}))\right]$ is ignored because aligning the learned latent space with a desired prior distribution in our training stage is unnecessary.

Two regularizers, originating from the first two terms, are calculated adaptively based on the samples utilized by each client during the local update.

With the reparameterization trick, the first term $E\int p(\overline{\bm{z}}|\bm{x}^{\prime}_n, \bm{\theta}_b)\log q(y_n|\overline{\bm{z}}, \bm{\theta}_b) d\overline{z}$ is estimated as $E_{\bm{x}^{\prime}} E_{\epsilon \sim N(0, I)}\left[\log q(\overline{y}^{\prime}|\bm{x}^{\prime}, \bm{\theta}_b, \epsilon)\right]$. Conditional on $\bm{\theta}_b$, the input $\bm{x}^{\prime}$ is encoded to $\bm{z}=f_{\bm{\theta}_b}\left(\bm{x}^\prime\right)$. After generating prototypes $\left\{\bm{v}_k\right\}_{k=1}^K$ from a batch of encodings, the \textit{prototype-based meta regularizer} $L_n$ is computed based on $p_{\bm{\theta}_b}\left(\overline{y}^{\prime}=k|\bm{x}^{\prime}\right)=\mathrm{softmax}\left(-d\left(\bm{z}, \bm{v}_k\right)\right)$ \cite{prototype-nips17}, where $d\left(\cdot\right)$ represents a distance measurement function, such as the Euclidean distance. The second term $E_{\bm{\theta}_b}E_{D^{\prime c}}\left[l^c(\bm{\theta}_b; \bm{x}^{\prime}, y^{\prime}) \right]$ aims to boost the representations extracted from $\bm{\theta}_b$ by introducing class distinguishability to benefit the personalized prediction. Under unsupervised learning, we propose a \textit{prototype-oriented contrastive regularizer} $L_p$ aiming to reduce prototype variances of the same classes from different augmented views of samples or vice versa.

Algorithm~\ref{algo:Calibre} presents the pseudocode for computing the loss of \emph{Calibre} (SimCLR) during the training stage. During each iteration of the local update, the computation of $L_n$ and $L_p$ relies on the augmented views of input samples generated by SSL approaches. Specifically, for one batch $B$ containing $N$ samples, the augmented dual-view samples are denoted as $I \equiv \left\{I_o, I_e\right\}$ where $I_o=\left\{\widehat{\bm{x}}_{2i}\right\}_{i=1}^N$ and $I_e=\left\{\widehat{\bm{x}}_{2i-1}\right\}_{i=1}^N$. To compute $l_n$, the prototypes $\left\{\bm{v}_k\right\}_{k=1}^{K_r}$ are generated based on the representations $\bm{z}$ of the sample from $I_e$. After assigning $I_o$ to these prototypes, the $l_n$ is the distance between each encodings of $I_o$ and its corresponding prototype, formulated as $L_n = \sum_{k \in K_r} \frac{-1}{ N_k }\sum_{j \in I^k_e}\mathrm{softmax}\left(-d\left(\bm{z}_j, \bm{v}_k\right)\right)$, where $I^k_e$ represents the samples assigned to prototype $k$ and $N_k = | I^k_e |$. Then, for the $L_p$, the prototypes $\left\{\bm{\nu}_{2k}\right\}_{k=1}^K$ and $\left\{\bm{\nu}_{2k-1}\right\}_{k=1}^{K_r}$ of $I_o$ and $I_e$ are computed based on the model outputs $\bm{h}_{2i-1}$ and $\bm{h}_{2i}$. These two views of prototypes then behave as contrastive samples to compute the $L_p$ following the formulation of NTXent loss \cite{simclr-icml20}. Finally the global loss $L = l_c + l_s + \alpha \left(l_p+ l_n\right)$, where $l_s$ depends on which SSL approach is used. In the case of \emph{Calibre} (SimCLR), the NTXent loss \cite{simclr-icml20} is used and denoted as $L_s$, as shown in Table~\ref{algo:Calibre}.

%% file: exps.tex
\section{Experimental Results}
\label{sec:exper}


\begin{figure*}[t]
  \centering
  \includegraphics[width=0.94\textwidth]{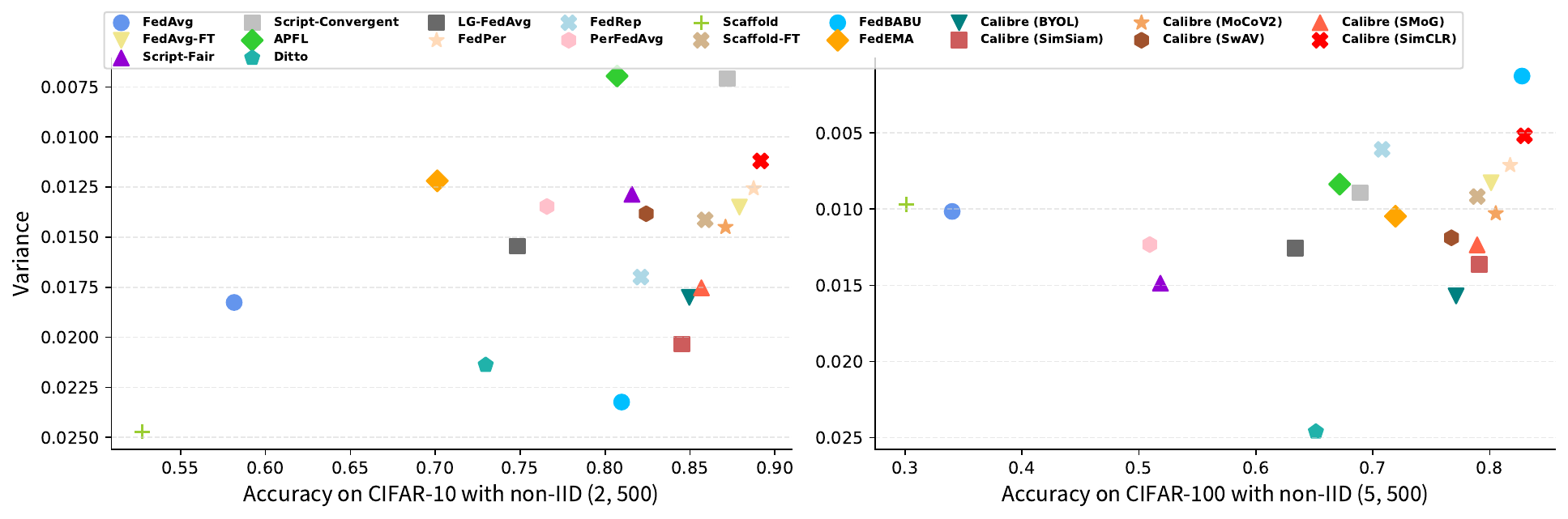}\\
  \includegraphics[width=0.94\textwidth]{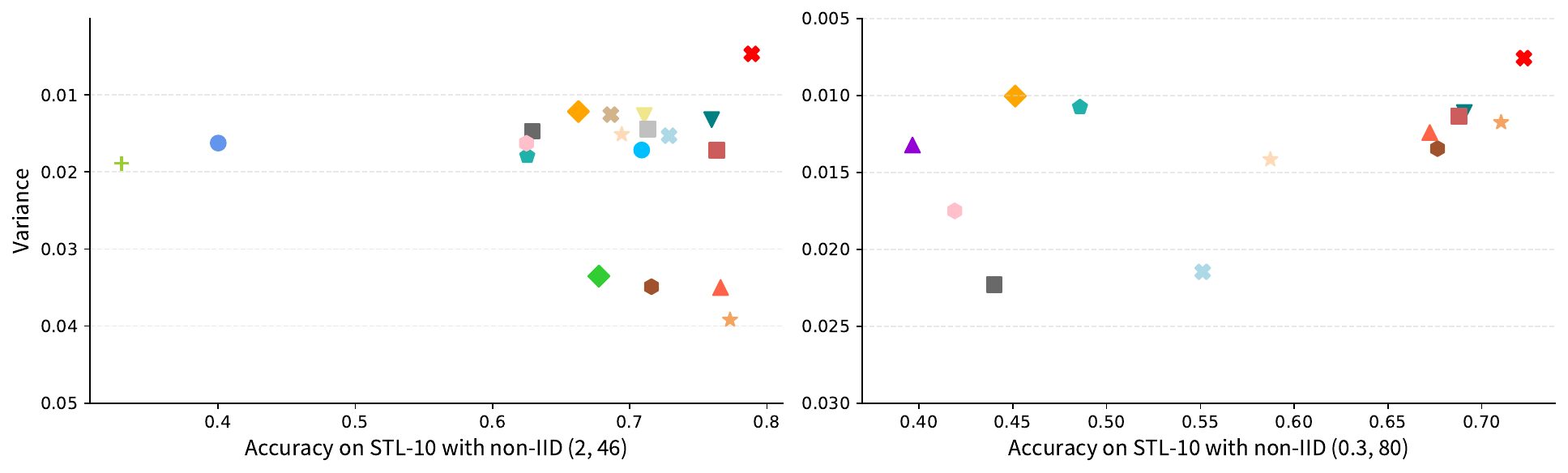}
  \caption{Comparison of Mean and Variance of Test Accuracy among $100$ clients across different non-i.i.d.~settings in the \texttt{CIFAR-10}, \texttt{CIFAR-100}, and \texttt{STL-10} datasets.}
  \vspace{-0.6cm}
  \label{Fig. acctovar}
\end{figure*} 

\subsection{Experimental Setup}
\label{resultsetup}

\textbf{Experimental platform}. All experiments are conducted on \href{https://github.com/TL-System/plato}{Plato}, an open-source research framework for federated learning. We utilize a single NVIDIA RTX A4500 GPU with 20 GB of CUDA memory. Additionally, one task is allocated 10 CPUs, each CPU having 3GB of memory.

\textbf{Datasets}. We perform the evaluation on three widely-used datasets, \texttt{CIFAR-10}, \texttt{CIFAR-100}, and \texttt{STL-10} \cite{stl10-aistats11}. The first three frequently utilized datasets comprise fully annotated samples. However, \texttt{STL-10} comprises $100,000$ unlabeled samples in addition to $5,000$ labeled training samples.

\textbf{Model settings}. We utilize the \textit{ResNet-18} network in experiments on \texttt{CIFAR-10}, \texttt{CIFAR-100} and \texttt{STL-10} datasets. In order to maintain a fair comparison, the fully-connected layers of both networks are substituted with a linear classifier. Consequently, for the \texttt{CIFAR-10}, \texttt{CIFAR-100}, and \texttt{STL-10} datasets, the input dimension of this classifier is set to $512$, corresponding to the number of classes as the output. We refer to the fully convolutional neural part as the \textbf{Encoder} while the linear classifier as the \textbf{Head}. Thus, in SSL approaches, the \textbf{Encoder} is the feature backbone, which also behaves as the global model exchanged between clients and the server. \emph{Calibre} (BYOL), \emph{Calibre} (SimCLR), \emph{Calibre} (MoCov2), \emph{Calibre} (Simsiam), \emph{Calibre} (SwAV), and \emph{Calibre} (SMoG) are built upon BYOL \cite{byol-nips20}, SimCLR \cite{simclr-icml20}, Simsiam \cite{simsiam-cvpr21}, MoCoV2 \cite{moco-cvpr20}, SwAv \cite{swav-nips20}, and SMoG \cite{smog-eccv22}, respectively. 

\textbf{Benchmark approaches}. Well-known federated learning approaches, including FedAvg \cite{fedavg-aistats17}, LG-FedAvg\cite{lgfedavg-nips19}, SCAFFOLD \cite{scaffold-icml20}, are included as the benchmark approaches. Specifically, FedAvg-FT and SCAFFOLD-FT refer to scenarios where the global model is initially trained using FedAvg and SCAFFOLD, respectively. Subsequently, the \textbf{Head} component of the model is fine-tuned based on the local dataset. Our experiments also contain advanced personalized FL approaches, including FedRep \cite{fedrep-icml21}, FedBABU \cite{fedbabu-iclr22}, FedPer \cite{fedper-arxiv19}, PerFedAvg \cite{perfedavg-nips20}, APFL \cite{apfl-arxiv20}. Calibre is also compared with Ditto \cite{ditto-icml21}, which achieves fairness through personalization. Additionally, we assess Calibre against the closely related work FedEMA \cite{fedema-iclr22} to provide a comprehensive evaluation. Finally, in addition to these approaches, we allow each client to train its personalized model (i.e., a linear classifier) separately based solely on their local datasets. Script-Convergent refers to the model trained until convergence, whereas Script-Fair corresponds to the model trained after $10$ epochs.

\textbf{Learning settings}. We have a total of $100$ clients participating in training the global model for $200$ rounds. In addition, there are $50$ novel clients that are excluded from the training process. In each round, $10$ clients are randomly selected to perform $3$ epochs of local update. After $200$ rounds of training, all $150$ clients will download the trained global model to perform the personalization based on the local dataset. Leveraging the trained global model, denoted as the \textbf{Encoder}, as the feature extractor, each client optimizes their personalized model for $10$ epochs using the SGD optimizer with a learning rate of $0.05$. The batch size is $32$. In the loss computation of \emph{Calibre}, we set the value of $\alpha$ as $0.3$. During empirical comparisons of SSL approaches, we chose a batch size of $256$. Although this batch size is relatively small, it does not impact the validity of our experiments or the resulting conclusions.
  
\textbf{Non-i.i.d.~settings}. \textit{Quantity-based label non-i.i.d.}, abbreviated as Q-non-i.i.d., means that each client owns data samples of a fixed number of labels. Thus, we control the heterogeneity of \texttt{CIFAR-10}, \texttt{CIFAR-100} and \texttt{STL-10} by assigning different numbers $S$ of classes per client, from among $K=10$ and $K=100$ total classes, respectively. Each client is assigned the same number of training samples, namely $\frac{D}{K}$, where $D$ is the total samples. It is denoted as $\left(S, \#samples \right)$. For \textit{distribution-based label non-i.i.d.}, abbreviated as D-non-i.i.d., each client uses a divided partition of the whole dataset, biased across labels according to the Dirichlet distribution with the concentrate $0.3$. We refer to this non-i.i.d.~setting as $\left(0.3, \#samples \right)$ for clarity and distinction.

\textbf{Embeddings of representations}. The qualitative results in our paper are depicted through a 2D embedding of the learned representations generated by the global model. During the personalization stage, each client forwards their local samples through the trained global model (\textbf{Encoder}) to extract features. The quality of these extracted features significantly impacts subsequent personalized learning. Thus, to gain deep insights, t-SNE is exploited as a dimensionality reduction technique to reduce the high-dimensional feature space to a 2D representation for visualization.

%% file: exps_detail.tex
\subsection{Accuracy}

\begin{figure*}[t]
    \centering
    \includegraphics[width=0.93\textwidth]{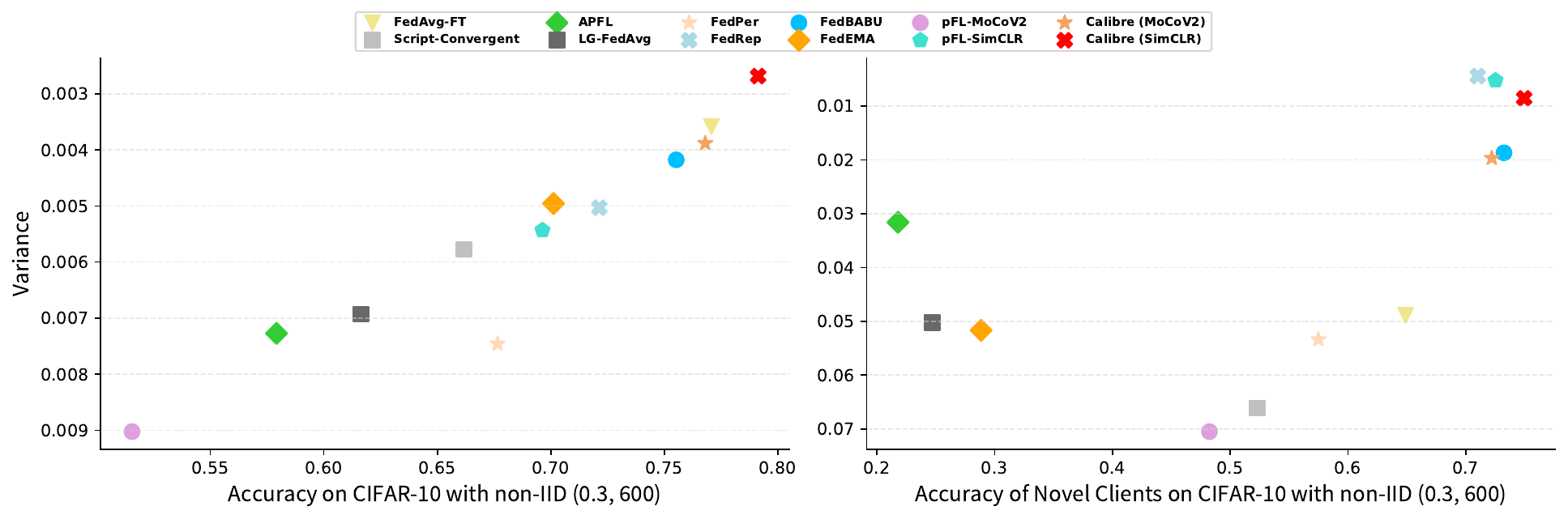}\\
    \includegraphics[width=0.93\textwidth]{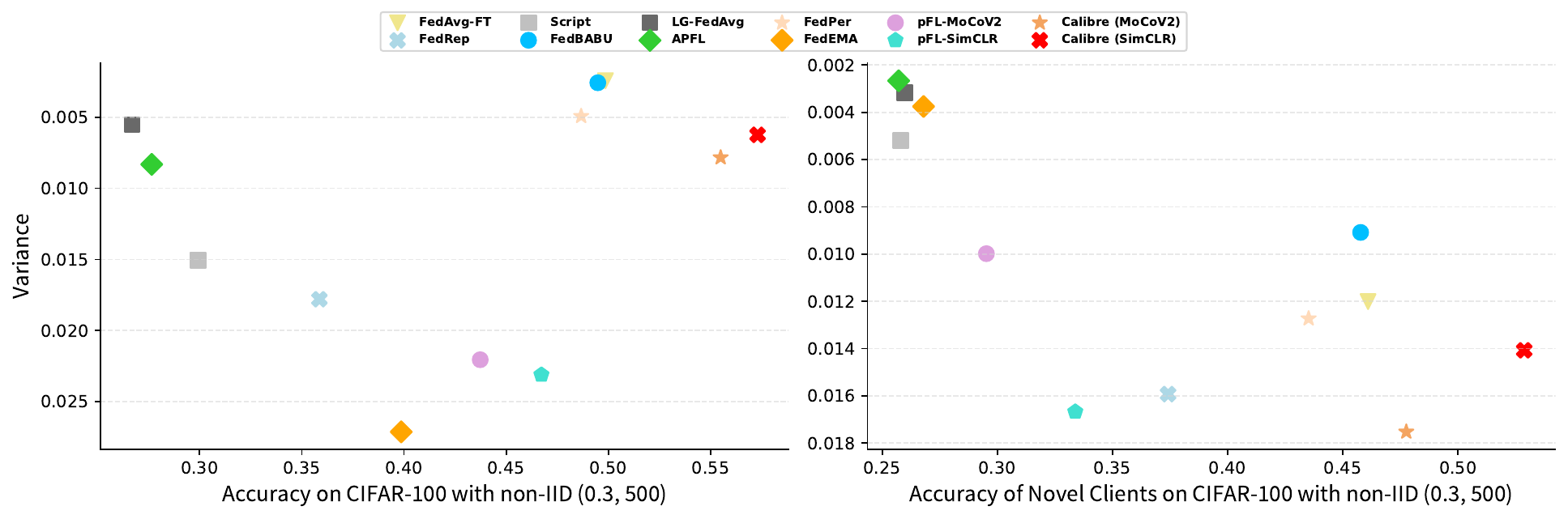}
    \caption{Comparison of Mean and Variance of test accuracy of 150 clients in \texttt{CIFAR-10} and \texttt{CIFAR-100} datasets under the \textit{distribution-based label non-i.i.d.}.}
    \vspace{-0.2cm}
    \label{Fig. dis}
\end{figure*} 
\begin{figure*}[ht]
    \centering
    \includegraphics[width=0.23\textwidth]{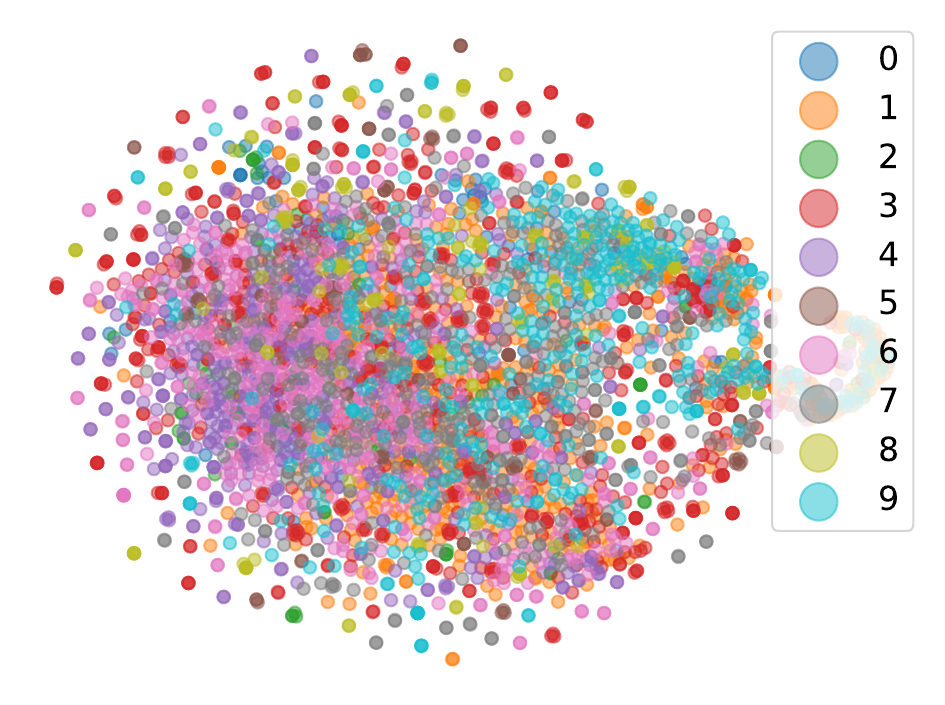}
    \includegraphics[width=0.23\textwidth]{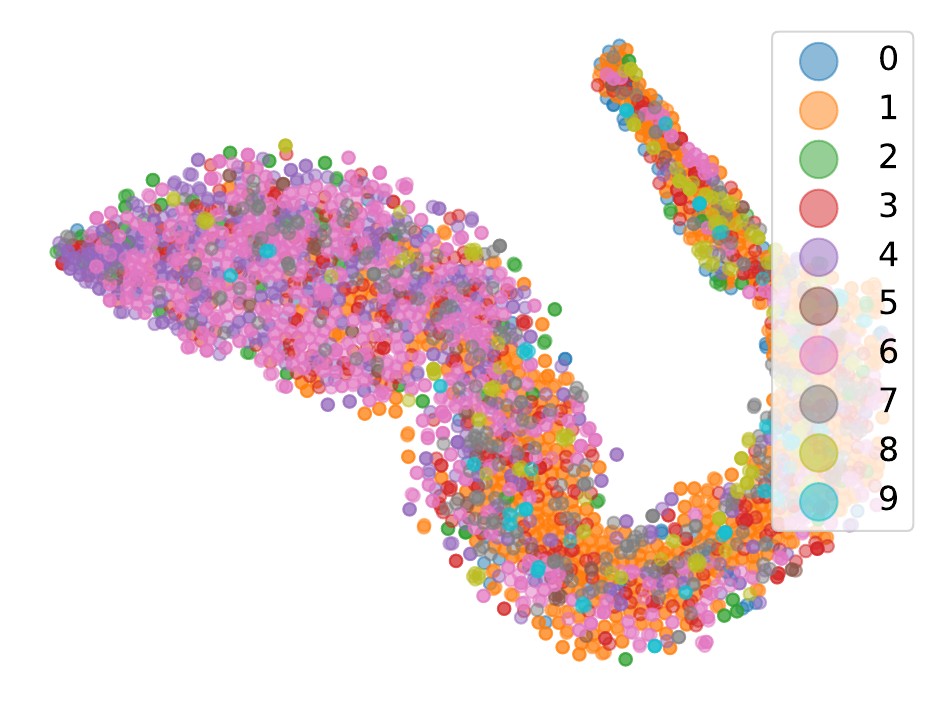}
    \includegraphics[width=0.23\textwidth]{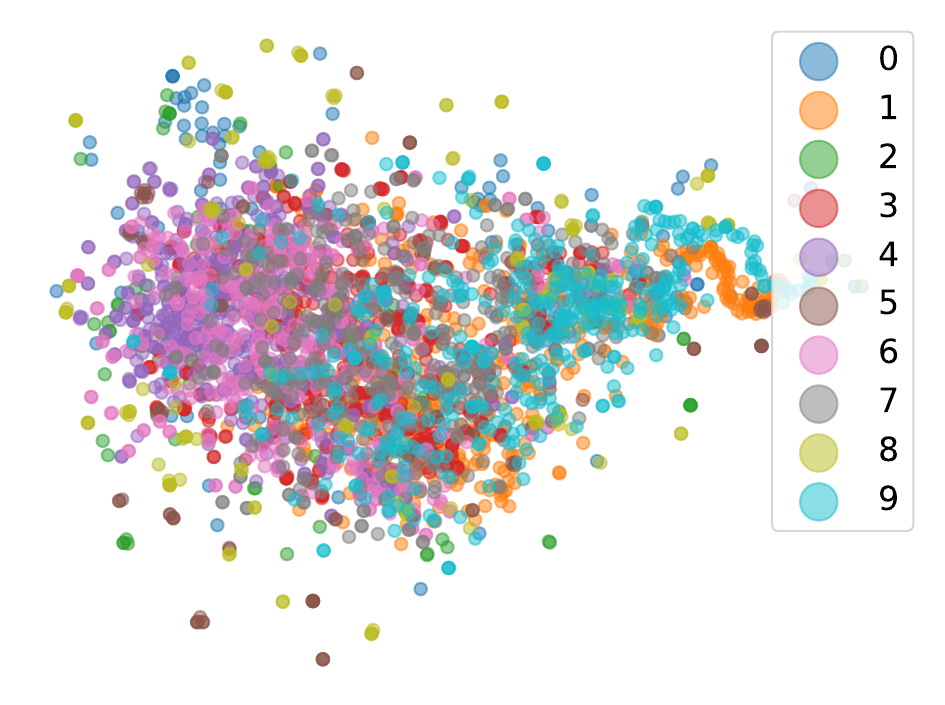}
    \includegraphics[width=0.23\textwidth]{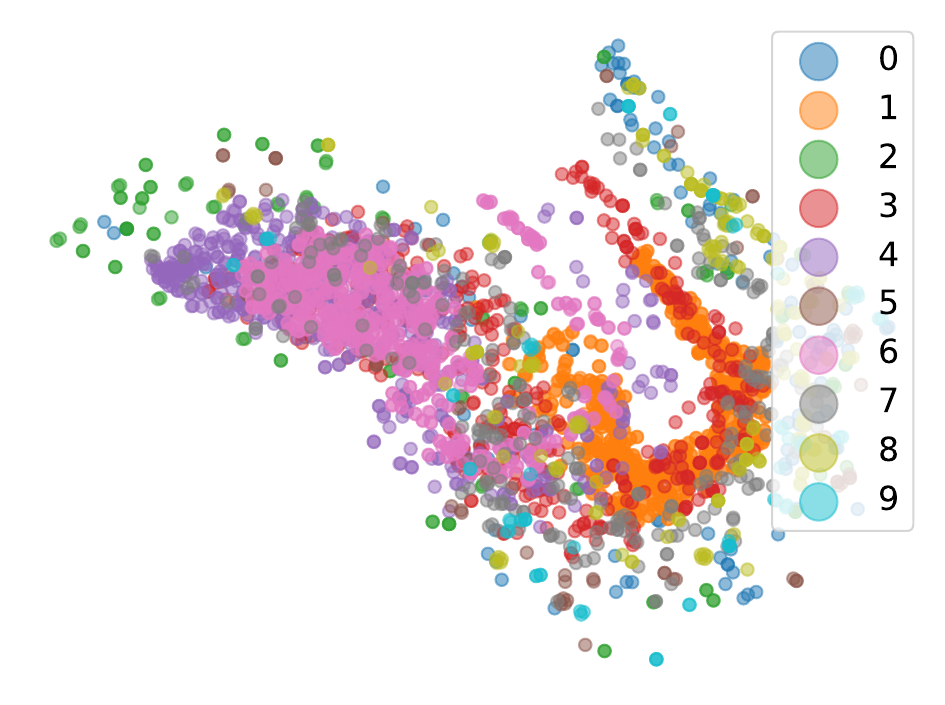}
    \caption{Illustrations of 2D t-SNE embeddings for representations obtained from encoders trained with pFL-SimSiam, pFL-MoCoV2, \emph{Calibre} (SimSiam), and \emph{Calibre} (MoCoV2), respectively. This experiment is conducted on the \texttt{CIFAR-10} dataset under D-non-i.i.d. with the concentration $0.3$. We collect representations from $6$ out of $100$ clients for visualization.}
    \vspace{-0.4cm}
\label{fig: qua2}
\end{figure*} 

As shown in \cref{Fig. acctovar}, and \cref{Fig. dis}, \emph{Calibre}, built upon four SSL methods, consistently obtains competitive accuracy. Especially, \emph{Calibre} (SimCLR) maintains the highest mean accuracy in various non-i.i.d.~settings of four datasets. In the Q-non-i.i.d. setting, when applied to the datasets \texttt{CIFAR-10}, \texttt{CIFAR-100}, and \texttt{STL-10}, Calibre utilizing the SimCLR learning structure demonstrates superior performance in terms of mean accuracy. It outperforms sub-optimal methods by $1.71\%$, $1.51\%$, and $6.03\%$ on the respective datasets. This is because features extracted by the global model trained with Calibre contain cluster information. Moreover, based on these transferable representations, the personalized model converges faster and can generalize better on the specific dataset, even when local samples are limited and imbalanced. Finally, any client can train a high-quality personalized model to gain high accuracy on its test set. 

While both non-i.i.d.~settings involve label skewness among clients, the D-non-i.i.d. setting presents a greater challenge than the Q-non-i.i.d. setting. This is because, in the D-non-i.i.d. setting, each client experiences sample skewness among classes, further complicating the learning process. Remarkably, under the D-non-i.i.d. setting for the \texttt{STL-10} dataset, \emph{Calibre} (SimCLR) significantly outperforms FedBABU by a considerable margin of $15.18\%$. The primary reason is that Calibre is designed to learn robust representations that target the generality-personalization tradeoff without explicitly relying on labels. Therefore, non-i.i.d.~data challenges present limited influence on the training process of our Calibre. And, Calibre is able to sufficiently learn from a large number of unlabeled samples in \texttt{STL-10} while other methods cannot.

When compared to traditional personalized FL methods, recent approaches like FedRep \cite{fedrep-icml21} and FedBABU \cite{fedbabu-iclr22}, which focus on learning a global encoder capable of extracting generic representations, demonstrate improved accuracy across clients and reduce accuracy divergence. The unique advantage of Calibre is the ability to avoid label skewness in representation learning while exploiting prototypes to maintain the distinguishability of features. The incorporation of Calibre into SimCLR and MoCoV2 during the federated training stage proves to be highly effective, as evidenced by the results obtained in the D-non-i.i.d. scenario on the \texttt{CIFAR-10} and \texttt{CIFAR-100} datasets shown in \cref{Fig. dis}. By utilizing the well-trained global model of \emph{Calibre}, the majority of clients can train personalized models of high quality, resulting in significantly improved mean accuracy. Specifically, in the demanding non-i.i.d.~scenarios of \texttt{CIFAR-10} and \texttt{CIFAR-100}, as illustrated in the subfigure of the first column in \cref{Fig. dis}, \emph{Calibre} (SimCLR) achieves superior mean accuracy compared to FedAvg-FT, surpassing it by $2.97\%$ and $7.11\%$, respectively.

\subsection{Fairness}

The fairness guarantee of our proposed Calibre has been substantiated through extensive experiments conducted on diverse datasets and a wide range of non-i.i.d.~settings, as evidenced by the results depicted in \cref{Fig. acctovar}, and \cref{Fig. dis}. Calibre consistently maintains a remarkably low accuracy variance across all four datasets, demonstrating its competitive performance. 

As shown in the first row of \cref{Fig. acctovar}, under the Q-non-i.i.d. in \texttt{CIFAR-10} and \texttt{CIFAR-100} datasets, the fairness performance of Calibre significantly approaches the optimal APFL and FedBABU, with the variance being only $0.0031$ and $0.0048$ higher than those of the two methods, respectively. This implies that the generated representations of Calibre face challenges in achieving robust generalization for 2-way classification tasks on certain clients. However, in the face of the challenging dataset, as exemplified by the second row of \cref{Fig. acctovar} for the \texttt{STL-10} dataset, Calibre demonstrates a substantial enhancement in fairness, surpassing the competitor method, FedEMA, by an impressive $31.1\%$ in percentage.

Notably, the advantage of Calibre in achieving fairness can be verified in challenging D-non-i.i.d. cases where clients exhibit varying degrees of label imbalance in their samples. As illustrated in the first column of \cref{Fig. dis}, \emph{Calibre}, built upon SimCLR and MoCoV2, trains the model to capture a fair condition that the accuracy of most clients remains a higher value. Consequently, Calibre achieves a remarkable reduction of $23.8\%$ in variance compared to FedAvg-FT. We contend that training the global encoder in an unsupervised manner naturally captures generic representations, particularly under D-non-i.i.d. scenarios. Although such representations can contribute fairly to all clients, the sparse class information in the representations may result in suboptimal mean accuracy. Calibre resolves this issue by calibrating the pure SSL representations using the concept of prototypes, leading to significant improvements.

\begin{table}\footnotesize\centering
    
    \caption{Test Accuracy $mean\pm std$ on the non-i.i.d.~setting with  $\left(2, 500 \right)$ for the \texttt{CIFAR-10} dataset. The checkmark (\checkmark) indicates that the corresponding term is included in Calibre.}
    \begin{adjustbox}{max width=\columnwidth}
        \begin{tabular}{ccc|l|l|l}
        \toprule
        $L_n$        & $L_p$        & \emph{Calibre} (SimCLR)    & \emph{Calibre} (SwAV)    & \emph{Calibre} (SMoG)  \\ \midrule \midrule
                     &              & $54.67\pm14.32$     & $85.03\pm15.1$    & $86.19\pm11.32$ \\ \midrule
                     & \checkmark   & $73.58\pm10.13$     & $84.76\pm12.5$    & $87.23\pm10.9$  \\ \midrule
        \checkmark   &              & $81.07\pm12.92$     & $79.31\pm15.73$   & $77.31\pm13.24$ \\ \midrule
        \checkmark   & \checkmark   & $89.16\pm10.58$     & $81.42\pm11.93$   & $80.07\pm11.2$  \\ \bottomrule
        \end{tabular}
    \end{adjustbox}
\label[type]{tab: ablastudy}
\vspace{-0.2cm}
\end{table}

\begin{figure*}[h]
    \centering
    \includegraphics[width=0.23\textwidth]{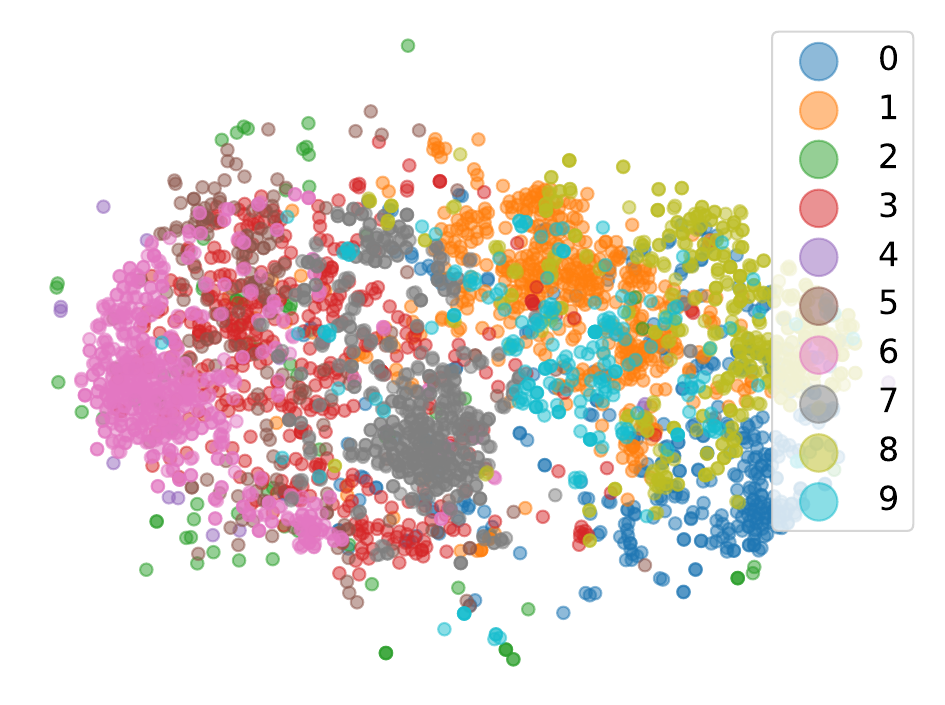}
    \includegraphics[width=0.23\textwidth]{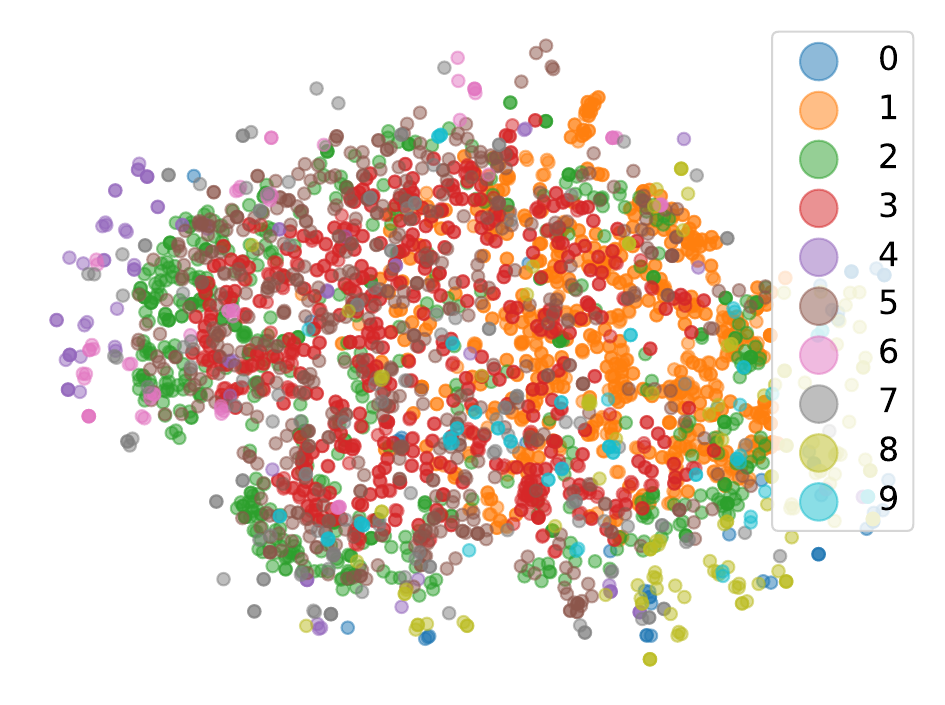}
    \includegraphics[width=0.23\textwidth]{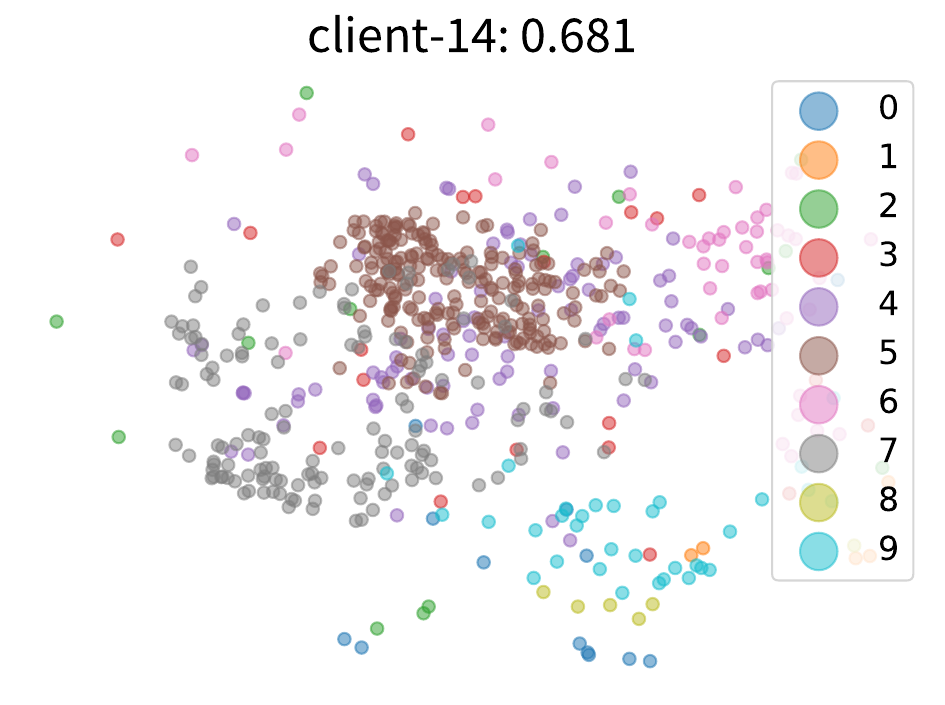}
    \includegraphics[width=0.23\textwidth]{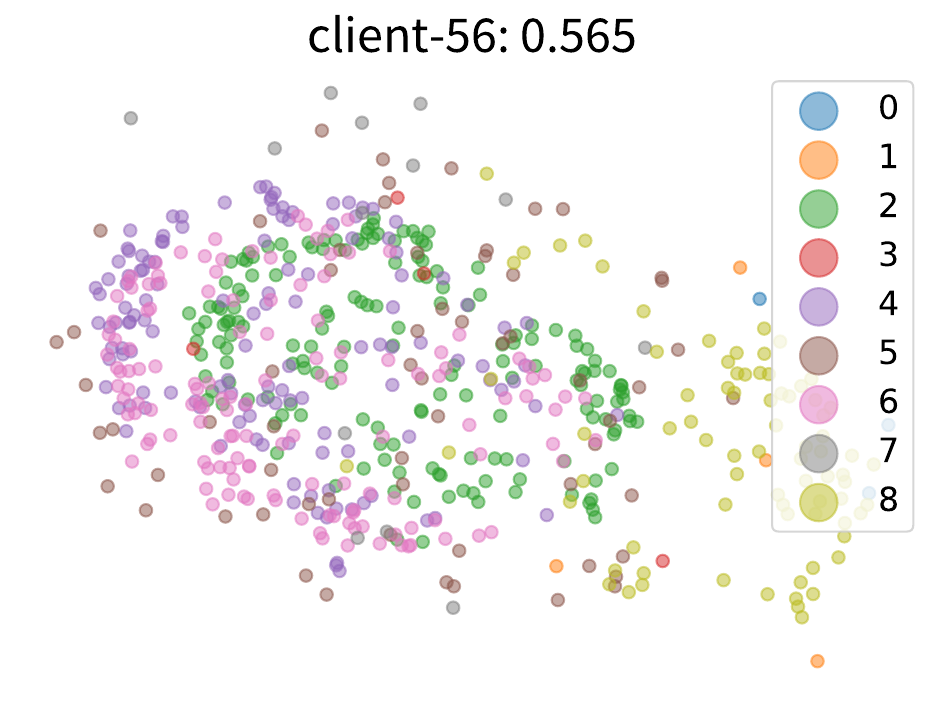}
    \caption{Illustrations of 2D t-SNE embeddings for representations obtained from encoders trained with \emph{Calibre} (SimCLR) and \emph{Calibre} (BYOL), respectively. This experiment is conducted on the \texttt{CIFAR-10} dataset under D-non-i.i.d. with the concentration $0.3$. Representations from $6$ out of $100$ clients are collected for visualization in the first two sub-figures. In comparison to \cref{fig: fuzzy1}, the last subfigures display the local representations for Client $14$ using \emph{Calibre} (SimCLR) and Client 56 using \emph{Calibre} (BYOL).}
    \vspace{-0.4cm}
\label{fig: qua3}
\end{figure*}

\begin{figure}[t]
    \centering
    \includegraphics[width=0.23\textwidth]{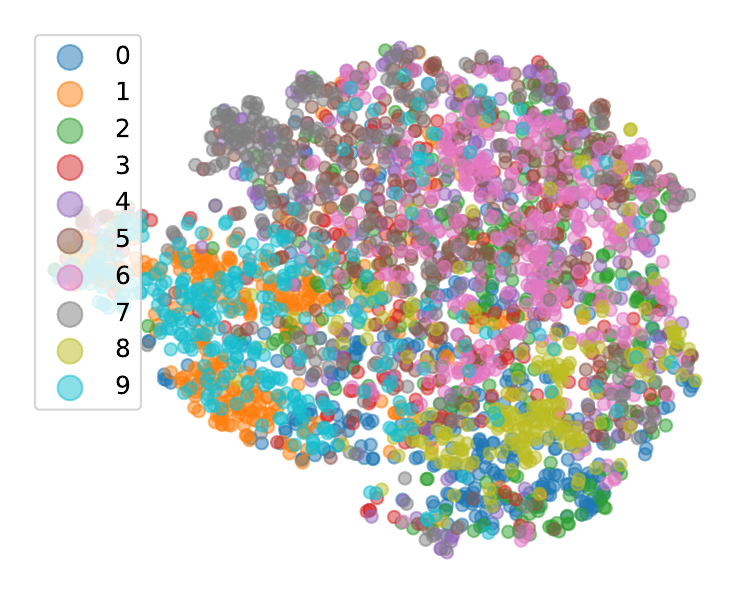}
    \includegraphics[width=0.23\textwidth]{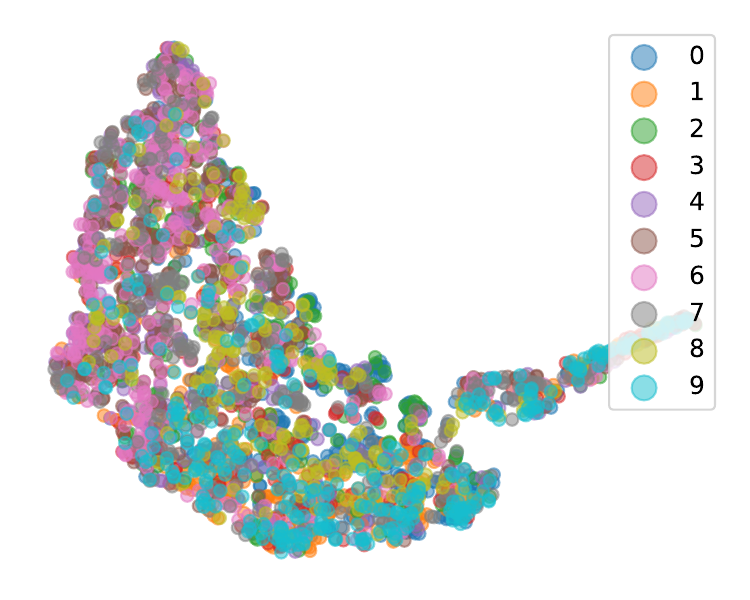}
    \includegraphics[width=0.23\textwidth]{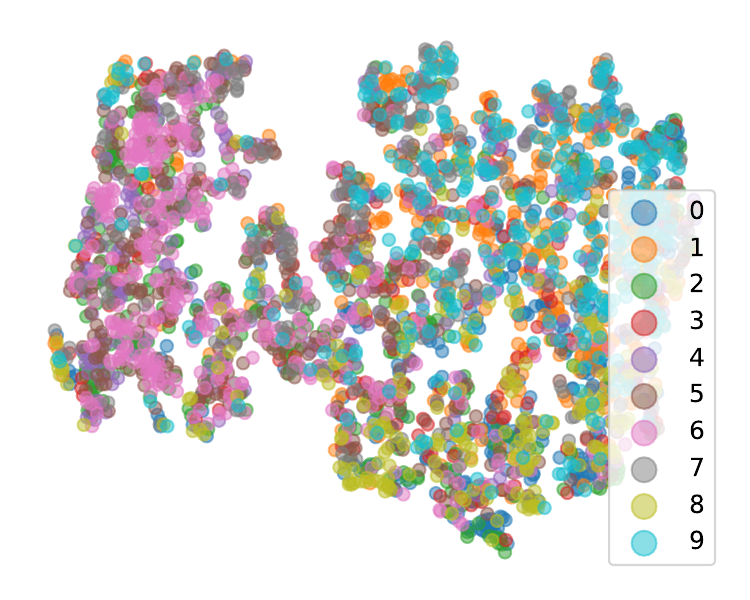}
    \includegraphics[width=0.23\textwidth]{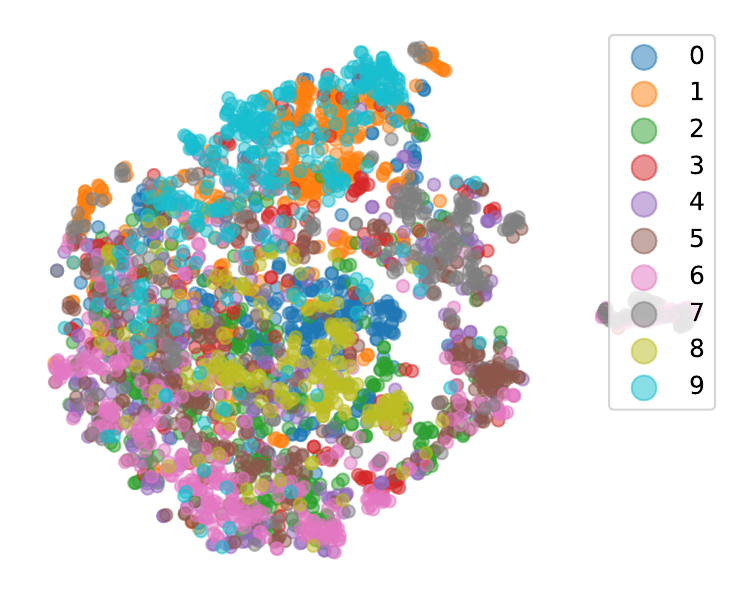}
    \includegraphics[width=0.23\textwidth]{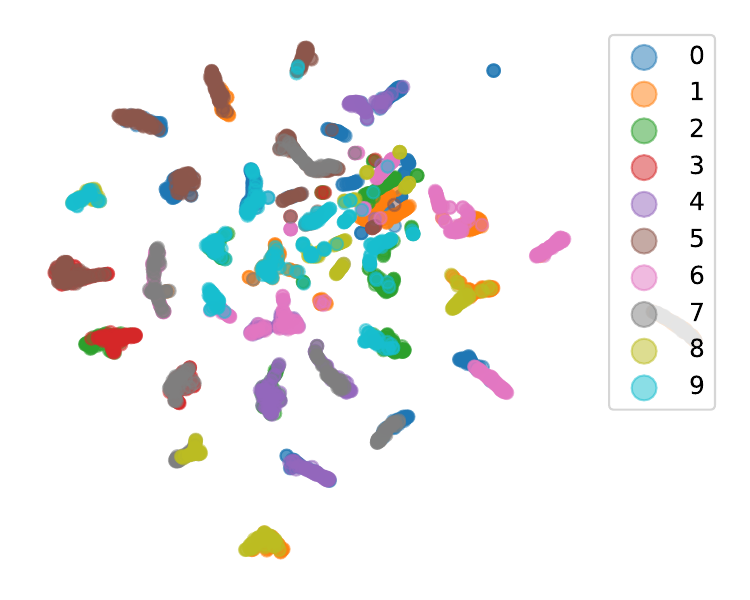}
    \includegraphics[width=0.23\textwidth]{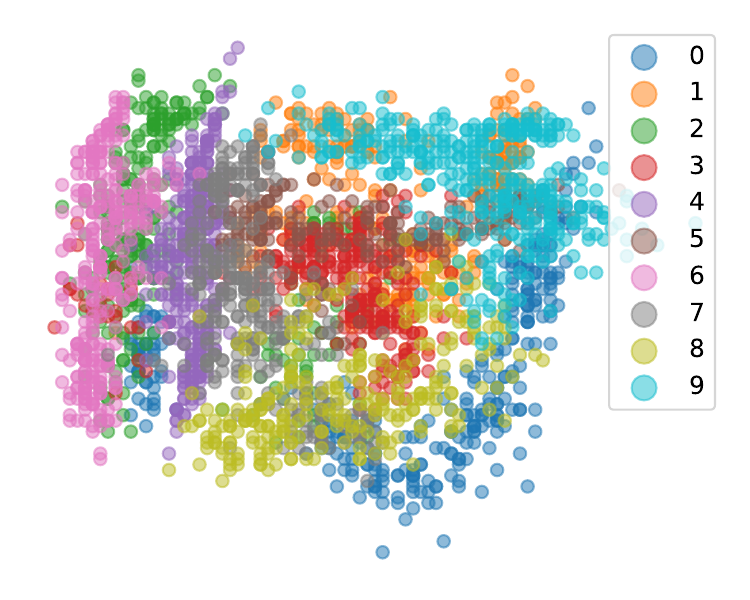}
    \caption{Illustrations of 2D embeddings obtained by utilizing the t-SNE on the learned representations of six methods. This experiment is conducted on the \texttt{CIFAR-10} dataset under the D-non-i.i.d. with the concentration $0.3$. From left top to right bottom, the representations are extracted from each client's local samples based on the encoder trained by FedAvg, FedRep, FedPer, FedBABU, LG-FedAvg and \emph{Calibre} (SimCLR), respectively.}
    \label{fig: cifar2d}
    \vspace{-0.35cm}
    \end{figure} 
    
    \begin{figure}[t]
    \centering
    \includegraphics[width=0.23\textwidth]{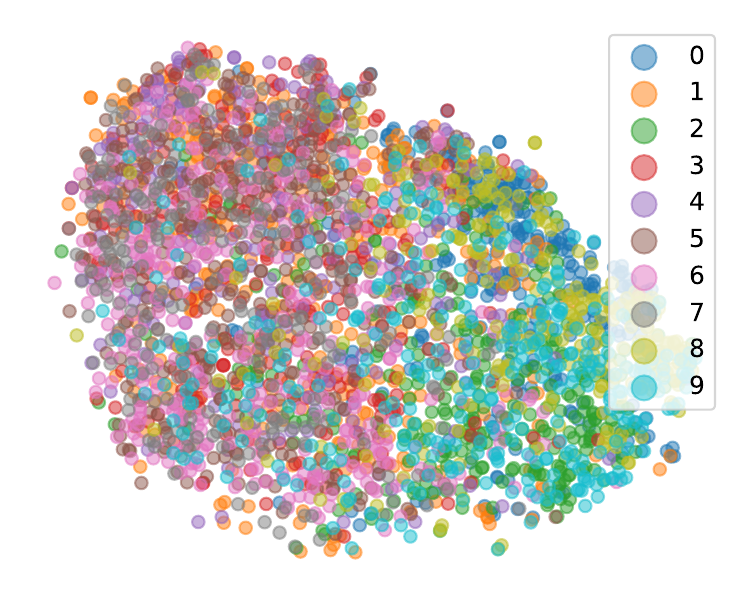}
    \includegraphics[width=0.23\textwidth]{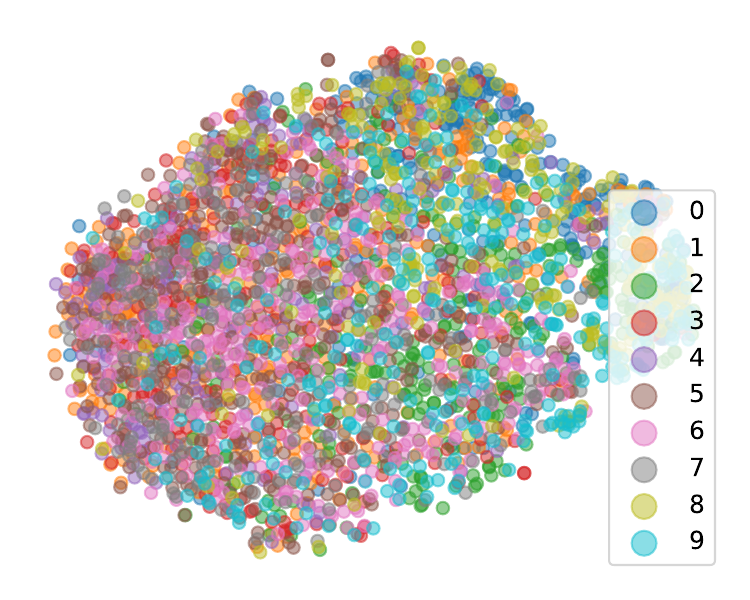}
    \includegraphics[width=0.23\textwidth]{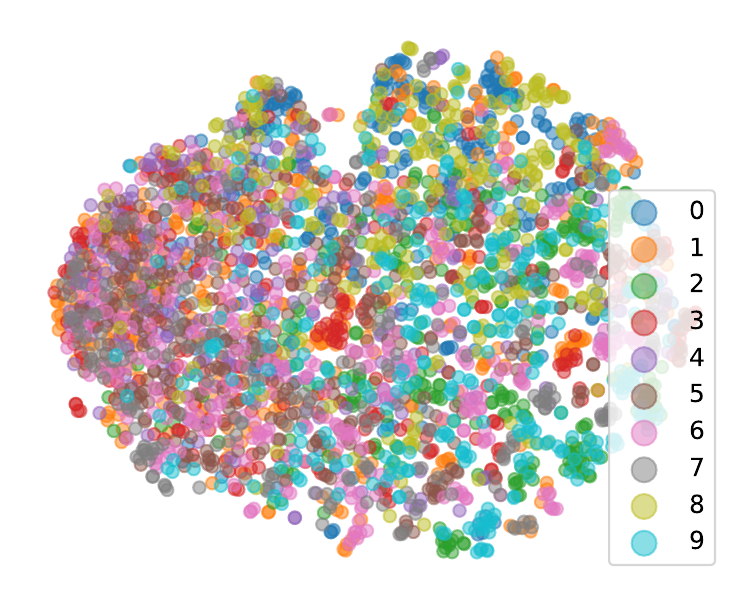}
    \includegraphics[width=0.23\textwidth]{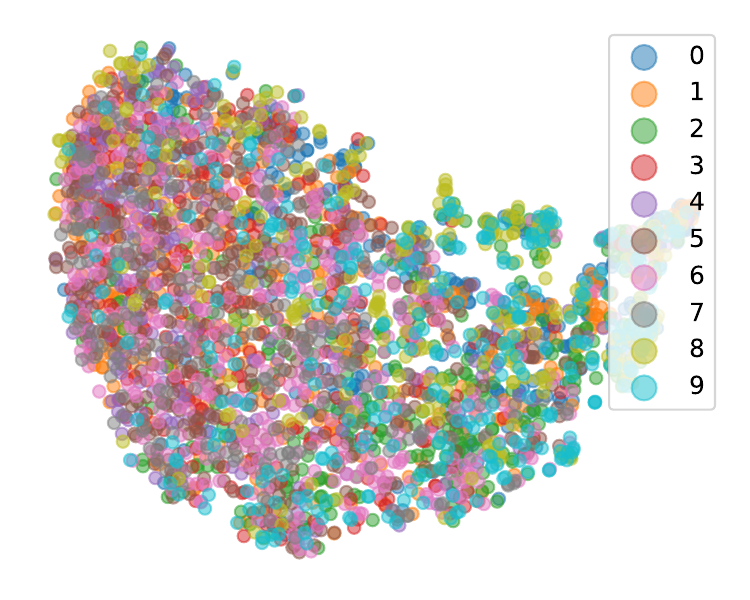}
    \includegraphics[width=0.23\textwidth]{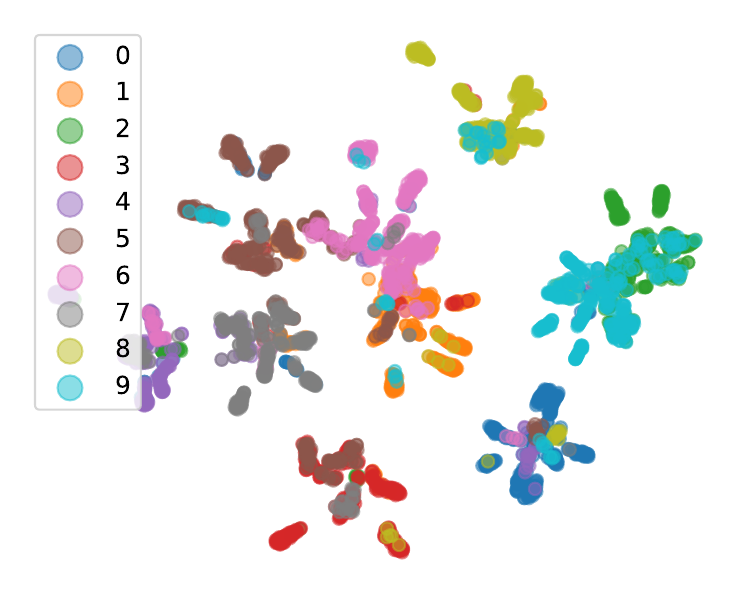}
    \includegraphics[width=0.23\textwidth]{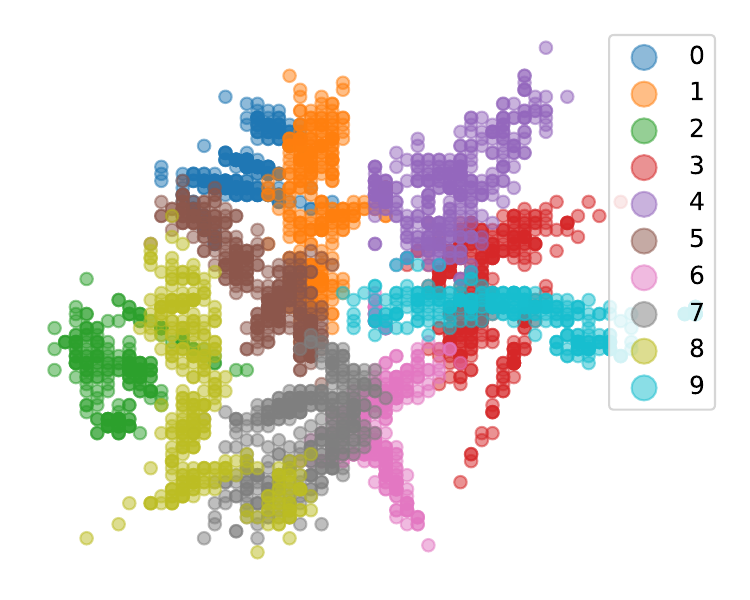}
    \caption{Illustrations of 2D embeddings obtained by utilizing the t-SNE on the learned representations of six methods. This experiment is conducted on the \texttt{STL-10} dataset under the Q-non-i.i.d. with each client assigned $2$ classes. From left top to right bottom, the representations are extracted from each client's local samples based on the encoder trained by FedAvg, FedRep, FedPer, FedBABU, LG-FedAvg and \emph{Calibre} (SimCLR), respectively.}
    \label{fig: stl2d}
    \vspace{-0.35cm}
    \end{figure}

\subsection{Performance on Novel Clients}

As shown in the second column of \cref{Fig. dis}, Calibre is able to achieve fair and accurate personalized learning even on novel clients that are unseen during training. In particular, when considering both the mean and variance of accuracy, Calibre achieves a suitable balance by achieving competitive fairness without compromising model performance. For example, the mean accuracy of \emph{Calibre} (SimCLR) outperforms the FedBABU by $2.2\%$ and $9.6\%$ in \texttt{CIFAR-10} and \texttt{CIFAR-100}, respectively. Meanwhile,  the corresponding variance on the \texttt{CIFAR-10} dataset for \emph{Calibre} (SimCLR) is $0.01$ lower than that of FedBABU. Despite the fact that the fairness of Calibre lags behind FedBABU by $0.0045\%$, its mean accuracy outperforms FedBABU by an even greater margin. A similar conclusion can be drawn when comparing Calibre with other approaches.

The reason for such outstanding performance of Calibre on novel clients is that Calibre trains the model to learn generic information while maintaining the ability to generate clusters for representations, all without depending on any client-specific information. Meanwhile, representations extracted by the trained model naturally contain prototypes of samples, guaranteeing a better class separation. Consequently, the trained global encoder can be readily employed by clients with any data distribution. 

\subsection{Rationale Behind the Superior Performance of \emph{Calibre} (SimCLR)}

Despite the significant improvements observed in SSL approaches calibrated by our \emph{Calibre}, \emph{Calibre} (SimCLR) consistently outperforms other methods by a substantial margin across all settings, as depicted in \cref{Fig. acctovar}, and \cref{Fig. dis}. We contend that the remarkable performance can be attributed to the NT-Xent objective function of SimCLR. This objective function simultaneously measures the inter- and intra-relations of positive and negative samples. By effectively bringing similar samples closer together while pushing dissimilar samples apart, it seamlessly cooperates with our regularizers, namely $L_p$ and $L_n$, to induce more robust prototypes. As a result, the clusters formed by different samples exhibit even clearer boundaries. As a comparison, the objectives of alternative methods, like BYOL, which exclusively utilize cosine similarity for measuring pairwise sample distances, could potentially undermine our desired tradeoff shown by \cref{theo: to_fa}. An illustrative example can be observed in \cref{fig: qua3}, which compares the representations of \emph{Calibre} (SimCLR) and \emph{Calibre} (BYOL). 

Experimental results, shown in Table~\ref{tab: ablastudy}, further present that the objective function of the SSL approach can work well in central learning but may show strong conflict with the regularizers of Calibre. The mean accuracy of \emph{Calibre} (SwAV) and \emph{Calibre} (SMoG) are lower than \emph{Calibre} (SimCLR) by $7.74\%$ and $9.09\%$, respectively. Nevertheless, the corresponding std is $1.35$ and $0.62$ higher. Therefore, although SwAV and SMoG inherently incorporate prototypes into the learning process, their objective cannot cooperate with Calibre. 

\subsection{Ablation Study}

In the ablation study presented in Table~\ref{tab: ablastudy}, we examine the contributions of various components, including the two regularizers $L_n$ and $L_p$, to the performance of Calibre. First, $L_n$ is more critical than $L_p$, as it is computed using the prototypes of encodings $\bm{z}$, which are utilized during personalized learning. Nevertheless, in the case of \emph{Calibre} (SwAV) and \emph{Calibre} (SMoG), where they construct their prototypes within the objective function, the importance of $L_p$ surpasses that of $L_n$. Incorporating $L_n$ into their training process may even have a detrimental effect. Notably, although SwAV and SMoG inherently incorporate prototypes in the learning process, they conflict with $L_n$, resulting in a considerable performance drop. Second, $L_p$ contributes more significantly to fairness. For example, for \emph{Calibre} (SimCLR), \emph{Calibre} (SwAV), and \emph{Calibre} (SMoG), containing $L_p$ in the objective leads to a lower variance. 

\subsection{Qualitative Results}

\cref{fig: qua2} and \cref{fig: qua3} showcase the 2D embeddings of SSL representations after being calibrated by our proposed Calibre. In comparison to representations learned by pFL-SimSiam and pFL-MoCoV2, which exhibit limited generic information and indistinct boundaries, the Calibre representations from multiple clients demonstrate clear clusters with refined class boundaries, as visually depicted in \cref{fig: qua2}. Significantly, when comparing \cref{fig: qua3} with \cref{fig: fuzzy0}, Calibre effectively addresses the issue of fuzzy cluster boundaries across clients discussed in subsection \ref*{prob: fuzzy}. For instance, BYOL learns the relationship between positive and negative samples, but after the application of \emph{Calibre}, it can capture clear cluster information for individual clients. Moreover, as illustrated in \cref{fig: qua3}, Calibre substantially alleviates the problem depicted in \cref{fig: fuzzy1}, which displays fuzzy cluster boundaries within each client.

Lastly, to facilitate a direct comparison of representations, we analyze the 2D embeddings of FedAvg, FedRep, FedPer, FedBABU, LG-FedAvg, and \emph{Calibre} (SimCLR) using samples from the \texttt{CIFAR-10} and \texttt{STL-10} datasets. As shown in \cref{fig: cifar2d} and \cref{fig: stl2d}, it is obvious that representations of \emph{Calibre} (SimCLR) consistently present clear clusters, which results in better boundaries for easier class separation. 

%% file: main.bbl
\begin{thebibliography}{10}
\providecommand{\url}[1]{#1}
\csname url@samestyle\endcsname
\providecommand{\newblock}{\relax}
\providecommand{\bibinfo}[2]{#2}
\providecommand{\BIBentrySTDinterwordspacing}{\spaceskip=0pt\relax}
\providecommand{\BIBentryALTinterwordstretchfactor}{4}
\providecommand{\BIBentryALTinterwordspacing}{\spaceskip=\fontdimen2\font plus
\BIBentryALTinterwordstretchfactor\fontdimen3\font minus
  \fontdimen4\font\relax}
\providecommand{\BIBforeignlanguage}[2]{{%
\expandafter\ifx\csname l@#1\endcsname\relax
\typeout{** WARNING: IEEEtran.bst: No hyphenation pattern has been}%
\typeout{** loaded for the language `#1'. Using the pattern for}%
\typeout{** the default language instead.}%
\else
\language=\csname l@#1\endcsname
\fi
#2}}
\providecommand{\BIBdecl}{\relax}
\BIBdecl

\bibitem{fedavg-aistats17}
B.~McMahan, E.~Moore, D.~Ramage, S.~Hampson, and B.~A. y~Arcas,
  ``{Communication-Efficient Learning of Deep Networks from Decentralized
  Data},'' in \emph{Proc.~20th International Conference on Artificial
  Intelligence and Statistics (AISTATS)}.\hskip 1em plus 0.5em minus
  0.4em\relax PMLR, 2017, pp. 1273--1282.

\bibitem{qffl-iclr20}
T.~Li, M.~Sanjabi, A.~Beirami, and V.~Smith, ``{Fair Resource Allocation in
  Federated Learning},'' in \emph{Proc.~International Conference on Learning
  Representations (ICLR)}, 2020.

\bibitem{tpfl-tnnls22}
A.~Z. Tan, H.~Yu, L.~Cui, and Q.~Yang, ``{Towards Personalized Federated
  Learning},'' \emph{IEEE Transactions on Neural Networks and Learning Systems
  (TNNLS)}, 2022.

\bibitem{lgfedavg-nips19}
P.~P. Liang, T.~Liu, L.~Ziyin, N.~B. Allen, R.~P. Auerbach, D.~Brent,
  R.~Salakhutdinov, and L.-P. Morency, ``{Think Locally, Act Globally:
  Federated Learning with Local and Global Representations},'' in
  \emph{Proc.~IEEE/CVF Conference on Neural Information Processing Systems
  (NIPS)}, 2019.

\bibitem{fedmvt-ijcai20}
Y.~Kang, Y.~Liu, and T.~Chen, ``{Fedmvt: Semi-Supervised Vertical Federated
  Learning with Multiview Training},'' in \emph{Proc.~International Workshop on
  Federated Learning for User Privacy and Data Confidentiality in Conjunction
  with IJCAI (FL-IJCAI)}, 2020.

\bibitem{fedrep-icml21}
L.~Collins, H.~Hassani, A.~Mokhtari, and S.~Shakkottai, ``{Exploiting Shared
  Representations for Personalized Federated Learning},'' in
  \emph{Proc.~International Conference on Machine Learning (ICML)}.\hskip 1em
  plus 0.5em minus 0.4em\relax PMLR, 2021, pp. 2089--2099.

\bibitem{fedbabu-iclr22}
S.-Y.~Y. Jaehoon~Oh, Sangmook~Kim, ``{FedBABU: Towards Enhanced Representation
  for Federated Image Classification},'' in \emph{Proc.~International
  Conference on Learning Representations (ICLR)}, 2022.

\bibitem{fedema-iclr22}
W.~Zhuang, Y.~Wen, and S.~Zhang, ``{Divergence-Aware Federated Self-Supervised
  Learning},'' in \emph{Proc.~International Conference on Learning
  Representations (ICLR)}, 2022.

\bibitem{fedu-iccv21}
W.~Zhuang, X.~Gan, Y.~Wen, S.~Zhang, and S.~Yi, ``{Collaborative Unsupervised
  Visual Representation Learning from Decentralized Data},'' in
  \emph{Proc.~IEEE/CVF International Conference on Computer Vision (ICCV)},
  2021, pp. 4912--4921.

\bibitem{fedca-edgesys20}
B.~van Berlo, A.~Saeed, and T.~Ozcelebi, ``{Towards Federated Unsupervised
  Representation Learning},'' in \emph{Proc.~Third ACM International Workshop
  on Edge Systems, Analytics and Networking (EdgeSys)}, 2020, pp. 31--36.

\bibitem{metamo-iclr20}
M.~Yin, G.~Tucker, M.~Zhou, S.~Levine, and C.~Finn, ``{Meta-Learning without
  Memorization},'' in \emph{Proc.~International Conference on Learning
  Representations (ICLR)}, 2020.

\bibitem{stl10-aistats11}
A.~Coates, A.~Ng, and H.~Lee, ``{An Analysis of Single-Layer Networks in
  Unsupervised Feature Learning},'' in \emph{Proc.~Fourteenth International
  Conference on Artificial Intelligence and Statistics (AISTATS)}.\hskip 1em
  plus 0.5em minus 0.4em\relax JMLR Workshop and Conference Proceedings, 2011,
  pp. 215--223.

\bibitem{fedamp-aaai21}
Y.~Huang, L.~Chu, Z.~Zhou, L.~Wang, J.~Liu, J.~Pei, and Y.~Zhang,
  ``{Personalized Cross-Silo Federated Learning on Non-IID Data},'' in
  \emph{Proc.~Association for the Advancement of Artificial Intelligence
  (AAAI)}, 2021, pp. 7865--7873.

\bibitem{fedper-arxiv19}
M.~G. Arivazhagan, V.~Aggarwal, A.~K. Singh, and S.~Choudhary, ``{Federated
  Learning with Personalization Layers},'' \emph{arXiv preprint
  arXiv:1912.00818}, 2019.

\bibitem{revisitssl-CVPR19}
A.~Kolesnikov, X.~Zhai, and L.~Beyer, ``{Revisiting Self-Supervised Visual
  Representation Learning},'' in \emph{Proc.~IEEE/CVF Conference on Computer
  Vision and Pattern Recognition (CVPR)}, 2019, pp. 1920--1929.

\bibitem{byol-nips20}
J.-B. Grill, F.~Strub, F.~Altché, C.~Tallec, P.~Richemond, E.~Buchatskaya,
  C.~Doersch, B.~Avila~Pires, Z.~Guo, M.~Gheshlaghi~Azar \emph{et~al.},
  ``{Bootstrap Your Own Latent: A New Approach to Self-Supervised Learning},''
  \emph{Proc.~Advances in Neural Information Processing Systems (NIPS)},
  vol.~33, pp. 21\,271--21\,284, 2020.

\bibitem{simclr-icml20}
T.~Chen, S.~Kornblith, M.~Norouzi, and G.~Hinton, ``{A Simple Framework for
  Contrastive Learning of Visual Representations},'' in
  \emph{Proc.~International Conference on Machine Learning (ICML)}.\hskip 1em
  plus 0.5em minus 0.4em\relax PMLR, 2020, pp. 1597--1607.

\bibitem{simsiam-cvpr21}
X.~Chen and K.~He, ``{Exploring Simple Siamese Representation Learning},'' in
  \emph{Proc.~IEEE/CVF Conference on Computer Vision and Pattern Recognition
  (CVPR)}, 2021, pp. 15\,750--15\,758.

\bibitem{ditto-icml21}
T.~Li, S.~Hu, A.~Beirami, and V.~Smith, ``{Ditto: Fair and Robust Federated
  Learning through Personalization},'' in \emph{Proc.~International Conference
  on Machine Learning (ICML)}.\hskip 1em plus 0.5em minus 0.4em\relax PMLR,
  2021, pp. 6357--6368.

\bibitem{moco-cvpr20}
K.~He, H.~Fan, Y.~Wu, S.~Xie, and R.~Girshick, ``{Momentum Contrast for
  Unsupervised Visual Representation Learning},'' in \emph{Proc.~IEEE/CVF
  Conference on Computer Vision and Pattern Recognition (CVPR)}, 2020, pp.
  9729--9738.

\bibitem{rethinkssl-nips20}
Y.~Yang and Z.~Xu, ``Rethinking the value of labels for improving
  class-imbalanced learning,'' \emph{Proc.~Advances in neural information
  processing systems (NIPS)}, vol.~33, pp. 19\,290--19\,301, 2020.

\bibitem{fedproto-aaai22}
Y.~Tan, G.~Long, L.~Liu, T.~Zhou, Q.~Lu, J.~Jiang, and C.~Zhang, ``{Fedproto:
  Federated Prototype Learning Across Heterogeneous Clients},'' in
  \emph{Proc.~Association for the Advancement of Artificial Intelligence
  (AAAI)}, vol.~1, 2022, p.~3.

\bibitem{prototype-nips17}
J.~Snell, K.~Swersky, and R.~Zemel, ``{Prototypical Networks for Few-Shot
  Learning},'' \emph{Proc.~Advances in Neural Information Processing Systems
  (NIPS)}, vol.~30, 2017.

\bibitem{swav-nips20}
M.~Caron, I.~Misra, J.~Mairal, P.~Goyal, P.~Bojanowski, and A.~Joulin,
  ``Unsupervised learning of visual features by contrasting cluster
  assignments,'' \emph{Proc.~Advances in neural information processing systems
  (NIPS)}, vol.~33, pp. 9912--9924, 2020.

\bibitem{smog-eccv22}
B.~Pang, Y.~Zhang, Y.~Li, J.~Cai, and C.~Lu, ``Unsupervised visual
  representation learning by synchronous momentum grouping,'' in
  \emph{Proc.~European Conference on Computer Vision (ECCV)}.\hskip 1em plus
  0.5em minus 0.4em\relax Springer, 2022, pp. 265--282.

\bibitem{scaffold-icml20}
S.~P. Karimireddy, S.~Kale, M.~Mohri, S.~Reddi, S.~Stich, and A.~T. Suresh,
  ``{Scaffold: Stochastic Controlled Averaging for Federated Learning},'' in
  \emph{Proc.~International Conference on Machine Learning (ICML)}.\hskip 1em
  plus 0.5em minus 0.4em\relax PMLR, 2020, pp. 5132--5143.

\bibitem{perfedavg-nips20}
A.~Fallah, A.~Mokhtari, and A.~Ozdaglar, ``{Personalized Federated Learning
  with Theoretical Guarantees: A Model-Agnostic Meta-Learning Approach},''
  \emph{Proc.~Advances in Neural Information Processing Systems (NIPS)},
  vol.~33, pp. 3557--3568, 2020.

\bibitem{apfl-arxiv20}
Y.~Deng, M.~M. Kamani, and M.~Mahdavi, ``{Adaptive Personalized Federated
  Learning},'' \emph{arXiv preprint arXiv:2003.13461}, 2020.

\end{thebibliography}
